\documentclass[11pt]{article}

\usepackage{yx}

\title{Upper Counterfactual Confidence Bounds: a New Optimism Principle for Contextual Bandits}

\author{Yunbei Xu\footnote{Columbia University, New York, NY; Email: \texttt{yunbei.xu@gsb.columbia.edu}.}
\and
Assaf Zeevi\footnote{  Columbia University, New York, NY; Email: \texttt{assaf@gsb.columbia.edu}.}
}

\date{}
\begin{document}

\maketitle

\begin{abstract}

The principle of optimism in the face of uncertainty is one of the most widely used and successful ideas in multi-armed bandits and reinforcement learning. However, existing optimistic algorithms (primarily \texttt{UCB} and its variants) often struggle to deal with general function classes and large context spaces. In this paper, we study general contextual bandits with an offline regression oracle and propose a simple, generic principle to design optimistic algorithms, dubbed ``Upper Counterfactual Confidence Bounds'' (\texttt{UCCB}). The key innovation of \texttt{UCCB} is building confidence bounds in policy space, rather than in action space as is done in \texttt{UCB}. We demonstrate that these algorithms are provably optimal and computationally efficient in handling general function classes and large context spaces. Furthermore, we illustrate that the \texttt{UCCB} principle can be seamlessly extended to infinite-action general contextual bandits, provide the first solutions to these settings when employing an offline regression oracle.

\end{abstract}

\section{Introduction}
\subsection{Motivation. }
 Algorithms that rely on the ``optimism principle'' have been a major cornerstone in the study of multi-armed bandit (MAB) and reinforcement learning problems. Roughly speaking,  optimistic algorithms are those that choose a deterministic action at each round, based on some optimistic estimate of future rewards. Perhaps the most representative example is the celebrated Upper Confidence Bounds (\texttt{UCB}) algorithm and its many variants. 
Popularity of optimistic algorithms stems from their simplicity and effectiveness: 
the analysis of \texttt{UCB}-type algorithms are usually more straightforward than alternative approaches, so they have become the ``meta-algorithms" for more complex settings such as infinite-action settings and reinforcement learning. They  are also often  preferable to weighted allocations among actions because of the ability to discard sub-optimal actions and achieve superior instance-dependent empirical performances.

Despite their prevalent use in traditional bandit problems, existing \texttt{UCB}-type algorithms have a glaring drawback   in  contextual MAB settings: their regret  often scales with the cardinality of the context  space \citep{russo2013eluder, slivkins2019introduction}. Although \texttt{UCB}-type algorithms are provably optimal for the  special ``linear payoff" formulation \citep{chu2011contextual} and its generalized-linear variant \citep{li2017provably}, these formulations utilize special function classes rather than general ones for function approximation. Despite encouraging empirical observations \citep{bietti2018contextual}, optimism-based algorithms provably achieve sub-linear regret for  contextual bandits with general function classes only under restrictive  distributional assumptions \citep{foster2018practical}. 

Current algorithms for general contextual bandits significantly deviate from the optimism principle, as highlighted in works like ``Beyond \texttt{UCB}'' \citep{foster2020beyond}. The analysis of these algorithms, when utilizing an offline regression oracle, involves complex mathematical induction \citep{simchi2020bypassing}, and the approach for infinite-action settings using an offline regression oracle remains ambiguous. These considerations lead to the two central questions studied in this paper:
\begin{center}
    \it Q1: Can we make the optimism principle optimal and computationally efficient for contextual bandits with general function approximation? \\
    \it Q2: Can we solve contextual bandits with general function classes and infinite actions with an offline regression oracle? 
\end{center}

\subsection{Contributions}
In this paper, we propose what we believe to be the first optimistic algorithms that are provably optimal and computationally efficient for contextual bandits with general function classes and large context space. Interestingly, almost all existing solutions to general contextual bandits \citep{auer2002nonstochastic, dudik2011efficient, agarwal2012contextual, agarwal2014taming, foster2020beyond, simchi2020bypassing}, whether computationally efficient or not, rely on weighted, randomized allocations among actions at each round. We refer to these as "randomized algorithms" in the paper. Moreover, when employing an offline regression oracle, we propose the first solutions to general contextual  bandits with infinite actions, a setting we believe naturally illustrates the simplicity and universality of optimism-based algorithms.

Key components of our proposed \texttt{UCCB} principle include:

\paragraph{Systematic analysis of confidence bounds in policy space:} We provide a systematic analysis to build confidence bounds in the policy space, rather than in the action space as is done in \texttt{UCB}. This perspective is valuable for contextual bandits and reinforcement learning with general function classes and large context spaces, where we argue that confidence bounds should be analyzed in the policy space. To the best of our knowledge, this is the first principle to make optimism both optimal and computationally efficient for general contextual bandits.

\paragraph{ Perspective of contextual potential function:} We propose a potential function perspective to articulate the effectiveness of optimism in the contextual setting. The concept of ``contextual potential'' is valuable for the analysis of contextual problems, wherein the potential function takes an expectation over the context, yet the widely-used potential arguments for optimism algorithms remain valid under this expectation. This results in a very succinct and insightful analysis, avoiding the complex mathematical induction seen in earlier works \citep{simchi2020bypassing}.

\paragraph{Novel framework for infinite-action contextual bandits:} We introduce a novel framework for studying contextual bandits with general function classes and infinite actions. This includes the concepts of ``counterfactual action divergence'' and the complexity measure ``average decision entropy.'' These contributions provide the first solutions to infinite-action contextual bandits with general function classes using an offline regression oracle.

\subsection{Contextual bandits with general function classes}
The stochastic contextual bandit problem can be described as follows. Let $\A$ be the action set, and $\X$ be the  space of contexts that supports the distribution $\D_\X$ (e.g., $\X$ can be a subset of Euclidean space). For all  $x\in\X, a\in\A$, denote $\D_{x,a}$ a reward distribution determined by  context $x$ and action $a$. At each round $t=1,\dots,T$, the agent first observes a context $x_t$ drawn i.i.d. according to $\D_{\X}$. She then chooses an action $a_t\in\A$ based on $x_t$ and the history  $H_{t-1}$ generated by $\{x_i,a_i,r_i(x_i,a_i)\}_{i=1}^{t-1}$, and finally observes the  reward $r_t(x_t,a_t)$, which is conditionally independent and distributed according to the distribution $\D_{x_t,a_t}$. We assume the rewards take values in the interval $[0,1]$. 
An admissible contextual bandit algorithm \texttt{Alg} is a (possibly randomized) procedure that associates each realization of $\{H_{t-1},x_t\}$ with an action $a_t$ to employ at  round $t$.

Previous literature on contextual MAB problems can be sorted into two categories:  the realizable setting and the agnostic setting.  In the realizable setting,  the agent has access to a function class $\F$, with its members $f\in\F$ being mappings from $\X\times\A$ to $[0,1]$. The following is referred to as the realizability condition \citep{ agarwal2012contextual, foster2018practical, foster2020beyond, simchi2020bypassing} :
\begin{assumption}[{\bf realizability}]\label{asm realizability} There exists $f^*\in\F$ such that for all $t\geq 1$, $x\in\X$, $a\in\A$, 
the conditional mean reward, $\E[r_t(x_t,a_t)|x_t=x,a_t=a]$, is equal to $f^*(x,a)$.
\end{assumption}
We call a mapping $\pi:\X\rightarrow\A$ from the context space $\X$ to the action set $\A$  a ``policy." (Those mappings may be referred to more precisely as ``deterministic stationary policies;" in this paper we often just refer to them as ``policies" with slight abuse of terminology.) Let $\pi_{f^*}$, defined by
$
      \pi_{f^*}(x)=\argmax f^*(x,a)
$,  be the ``ground truth" optimal policy. The cumulative (pathwise) regret of a contextual bandit algorithm \texttt{Alg} compared with the optimal policy $\pi_{f^*}$ after $T$ rounds is 
$$
 \text{Regret}(T, \texttt{Alg}):=\sum_{t=1}^T(r_t(x_t,\pi_{f^*}(x_t))-r_t(x_t,a_t)),
$$
and the agent aims to minimize this cumulative regret. 
The agnostic setting \citep{auer2002nonstochastic, langford2007epoch, dudik2011efficient, agarwal2014taming}, on the other hand, does not make such realizability assumption; instead, algorithms are compared with the best policy within a given policy class.
 In this paper we  focus on the realizable setting  which lends itself more naturally to the design of optimism-based algorithms.

We present some examples of the realizable setting. In the initial parts of the paper, one can think of
$\A$ as the integer set $\{1, \dots, K\}$, which we generalize later on. The most well-studied  contextual MAB problems are simple variants of the ``linear payoff" model \citep{chu2011contextual, li2017provably}
\begin{align}\label{eq: linear contextual bandit}
    \F=\{f:f(x,a)=\theta^Tx_a,  \theta\in\Theta\}, \quad\Theta, \X\subseteq{\R}^d, x=(x_a)_{a\in\{1,\dots,K\}}.
\end{align}
One motivation  towards general function classes is to encompass models of the form 
\begin{align}\label{eq: broad class context}
    \F=\{f: f(x,a)=g_a(x), \quad  g_a\in \mathcal{G}, a\in\{1,\dots,K\}\},
\end{align}where  parameters of $g_a:\X\rightarrow \R$ can be distinct for different actions \citep{krishnamurthy2019active}; it is also desirable to handle complex nonlinear models (such as neural networks) which are much more expressive than their linear counterparts.

On the computation side, we make the rather benign assumption the the agent has access to a pre-specified least square  oracle over $\F$. Formally, after the agent inputs the historical data $\{x_i,a_i,r_i(x_i,a_i)\}_{i=1}^{t-1}$, the  least square oracle outputs a solution $\widehat{f}_t\in\F$ that provides the best fit, namely,
\begin{align}\label{eq: least square}
   \widehat{f}_t\in\argmin_{f\in\F}\sum_{i=1}^{t-1}(f(x_i,a_i)-r_i(x_i,a_i))^2.
\end{align}
This is the simplest optimization oracle assumed in the contextual bandit literature. We assume the least square oracle to be deterministic, for simplicity, as there may be  multiple solutions to \eqref{eq: least square}.

\subsection{Related literature}
We review related works in several areas.

\paragraph{Contextual bandits with general function classes (and finite actions).}

In this paper, we focus on the realizable contextual bandits setting with offline regression oracles. In this setting,  the minimax  regret of stochastic contextual bandits is $O(\sqrt{KT\log|\F|})$\footnote{we adopt non-asymptotic big-oh notation: for functions $h_1,h_2$, $h_1=O(h_2)$ if there exists constant $C>0$ such that $h_1$ is dominated by $Ch_2$ with high probability (omitting $\log\frac{1}{\delta}$ factors);  $h_1=\tilde{O}(h_2)$ if $h_1=O(h_2\max\{1, \text{polylog}(h_2)\}).$} for a general finite function class $\F$. 
In \cite{agarwal2012contextual} the non-efficient algorithm \texttt{Regressor Elimination} was proposed to achieve  optimal regret. 

The seminal work ``Beyond \texttt{UCB}'' \citep{foster2020beyond} proposed an optimal and  oracle-efficient algorithm called \texttt{SquareCB}. While the use of the online regression oracle is very elegant, the online regression oracle  is only computationally efficient for specific function classes. In contrast, we are interested in the weaker and more practical offline least square oracle \eqref{eq: least square}, which is commonly used in statistical learning.

The open problem of optimal realizable contextual MAB  with an offline least square oracle was first solved by  \cite{simchi2020bypassing}, with a  randomized algorithm called \texttt{FALCON}. One very inspiring aspect of \texttt{FALCON} is that weighted allocation in policy space can be implicitly achieved by weighted allocation over actions under the realizability assumption---this implication was referred to as ``bypassing the monster" in \cite{simchi2020bypassing}. This motivates the investigation in the present paper that  considers implicit optimization over policies when designing optimistic algorithms. Unlike the \texttt{FALCON} algorithm, our approach is predicated on the optimism principle and computing counterfactual action trajectories.

 In the agnostic contextual bandits setting,
the minimax  regret  is $O(\sqrt{KT\log|\bar{\Pi}|})$ given a finite policy class $\bar{\Pi}\subset\Pi$. The earliest optimal solution to agnostic contextual bandits is the \texttt{EXP4} algorithm \citep{auer2002nonstochastic} whose computation is linear in $|\bar{\Pi}|$. There are two optimal offline-oracle-efficient randomized algorithms using the  cost-sensitive classification (\texttt{CSC}) oracle:  \texttt{Randomized UCB} \citep{dudik2011efficient} and \texttt{ILOVETOCONBANDIT} \citep{agarwal2014taming}.

\paragraph{Variants of \texttt{UCB} for particular contextual bandit problems.}    Variants of \texttt{LinUCB} are well-known to be regret-optimal and efficient for simple variants of \eqref{eq: linear contextual bandit}. However, for general function classes, existing variants of \texttt{UCB}   typically have their regret scaling with  $|\X|$ \citep{russo2013eluder, slivkins2019introduction}, except under strong assumptions on the data distribution \citep{foster2018practical}. 
\texttt{UCB} has also been used as a subroutines in  contextual bandits when the  functions in $\F$ admit  smoothness  or Lipchitz continuity over $\X$ \citep{rigollet2010nonparametric, slivkins2014contextual}. These works are usually based on discretization of $\X$.

\paragraph{Contextual bandits with general function classes and infinite actions.}
 \cite{abernethy2013large} studies how to reduce realizable contextual MAB with infinite actions to an online learning oracle called knows-what-it-knows (KWIK), but this oracle  is only known to exist for restricted function classes. \cite{foster2020beyond} studies how to combine general function classes with a linear action model (our illustrative example \eqref{eq: linear semi} in Section \ref{sec infinite}). However, their results crucially rely on the restrictive assumption that the action set $\A$ is the unit ball. The concurrent work \cite{foster2020adapting} (after the initial appearance of our preprint) studies this formulation with general action set, but they assume access to the stronger online regression oracle which is not computationally efficient in general (in contrast to our use of the weaker and more practical offline regression oracle).   \cite{krishnamurthy2019contextual} studies infinite-action contextual bandits in a quite general agnostic setting. Their formulation and results are quite different from ours, and they do not provide a computationally efficient algorithm.

Lastly, our proposed \texttt{UCCB} principle has been extended to Contextual Markov Decision Processes (CMDPs) in the follow-up works \citep{levy2022counterfactual, levy2023optimism}.
\subsection{Organization}
In Section \ref{sec uccb} we introduce the first optimal and efficient optimistic algorithm in the finite-action setting, dubbed ``\texttt{UCCB},'' and explain the key ideas underlying its principles.  In Section \ref{sec infinite} we introduce a unified framework for contextual bandits with infinite action spaces, and present several interesting examples for which our work gives rise to the first offline-regression-oracle-efficient solutions. In Section \ref{sec randomized}, we demonstrate that using the optimism principle in designing subroutine helps extending the analysis of \texttt{FALCON} to the infinite-action space when utilizing an offline regression oracle. In Section \ref{sec conclusion} we make conclusion and point out future directions.

\section{Upper counterfactual confidence bounds}\label{sec uccb}

\subsection{Introducing \texttt{UCCB}: two equivalent viewpoints}\label{subsec high level}
This subsection will describe the \texttt{UCCB} principle introduced in this paper from two equivalent viewpoints: 1) implicitly, it is an upper confidence bound rule in policy space; and 2) explicitly, it calculates the upper confidence bound via simulating {\it counterfactual} action trajectories rather than using the original action trajectory. For illustration purpose we focus on the finite-action setting where $\A=\{1,\dots,K\}$; extension to infinite-action spaces will be discussed later in Section \ref{sec infinite}.

\paragraph{\bf Implicit strategy: maximizing upper confidence bounds in policy space.}  Let $\Pi$ be the  {\it policy space} that contains all deterministic stationary policies $\pi:\X\rightarrow \{1,\dots,K\}$. The core idea of \texttt{UCCB} is to choose policies that maximize certain upper confidence bounds  in the policy space $\Pi$.  After initialization, for each round $t$, data $\{(x_i,a_i),r_i\}_{i=1}^{t-1}$ is sent to an offline least square oracle to compute the  estimator $ \widehat{f}_{t}\in\F$. Without the need to ``see" $x_t$, the agent selects the optimistic policy $\pi_t\in\Pi$ (which is a mapping from $\X$ to the action set $\{1,\dots, K\}$) such that
\begin{align}\label{eq: policy ucb}
    \pi_t\in\argmax_{\pi\in\Pi}\left\{ \E_x[\widehat{f}_t(x, \pi(x)]+{\E_x\left[\frac{\beta_t}{\sum_{i=1}^{t-1}{\mathds{1}\{\pi(x)=\pi_i(x)\}}}\right]}+\frac{K\beta_t}{t}\right\},
\end{align}
where the expectation $\E_x[\cdot]$ is  history-independent and taken with respect to the distribution $\D_{\X}$ (over the random context $x$), and $\beta_t$ is a parameter related to the complexity of the function class. (When there are multiple solutions to \eqref{eq: policy ucb}, we take $\pi_t$ to be the unique solution such that for all other solutions $\pi'$ to \eqref{eq: policy ucb} and all $x\in\X$, the index of the action $\pi_t(x)$ is smaller than the index of the action $\pi'(x)$.) Then the agent observes $x_t$ and selects the action $a_t=\pi_t(x_t)$. The right hand side of \eqref{eq: policy ucb} is an upper confidence bound on the true expected reward of $\pi$, because  we can prove that with high probability, for all  $\pi\in\Pi$,
\begin{align*}
    \big|\E_x[f^*(x,\pi(x))]-\E_x[\widehat{f}_t(x, \pi(x)]\big|\leq {\E_x\left[\frac{\beta_t}{\sum_{i=1}^{t-1}{\mathds{1}\{\pi(x)=\pi_i(x)\}}}\right]}+\frac{K\beta_t}{t}.
\end{align*}

\paragraph{\bf Explicit strategy: constructing confidence bounds via counterfactual actions.}  The distribution $\D_{\X}$ is unknown so there are both statistical and computational challenges in the optimization over  policies. However, since our proposed policy optimization problem \eqref{eq: policy ucb} is decomposable across contexts, there is an equivalent strategy where no explicit policy optimization is required:  at round $t$, after observing $x_t$, the agent selects the optimistic {\it action}
 \begin{align*}
     a_t\in\argmax_{a\in\{1,\dots,K\}} \left\{\widehat{f}_t(x_t,a)+\frac{\beta_t}{\sum_{i=1}^{t-1}\mathds{1}\{a=\widetilde{a}_{t,i}\}}\right\}
 \end{align*}
(ties broken by choosing the action with the smallest index), where  $\{\widetilde{a}_{t,i}\}_{i=1}^{t-1}$ is the {\it counterfactual action trajectory} for context $x_t$, defined as realizations of all past chosen policies $\{\pi_i\}_{i=1}^{t-1}$ on the context $x_t$. To recover the counterfactual actions, at round $t$, the agent runs an inner loop to sequentially generate $\widetilde{a}_{t,1},  \dots, \widetilde{a}_{t,t-1}$: for $i=1, \dots, t-1$,
 \begin{align*}
         \widetilde{a}_{t,i}\in\argmax_{a\in\{1,\dots,K\}}\left\{ \widehat{f}_i(x_t,a)+\frac{\beta_i}{\sum_{j=1}^{i-1}\mathds{1}\{a=\widetilde{a}_{t,j}\}}\right\},
         \end{align*}
(ties are broken by choosing the action with the smallest index). Our approach is clearly quite distinct from  previous variants of \texttt{UCB}, as we construct confidence bounds by using {\it simulated counterfactual actions} rather than  using the {\it actual selected actions}.
 
The $\texttt{UCCB}$ principle leads to provably efficient optimism-based algorithms for general function classes: their regret bounds do not scale with the cardinality of the context spaces, and the required offline least square oracle is feasible for most natural function classes.

\subsection{The algorithm}
Following previous works \citep{agarwal2012contextual, foster2018practical, foster2020beyond, simchi2020bypassing}, we start by assuming $\A=\{1,\dots,K\}$, $|\F|<\infty$, and target  the ``gold standard" in this area---$\text{Regret}(T, \texttt{Alg})\leq\tilde{O}(\sqrt{KT\log |\F|})$, which emphasises the logarithmic scaling in the cardinality $|\F|$. This is mainly for illustrative purposes, and we discuss extensions to infinite function classes in Section \ref{subsec infinite function class}. A relatively new setting which has essentially not been explored is the infinite-action setting, which we will discuss in Section \ref{sec infinite}.
\begin{algorithm}[htbp]
\caption{ Upper Counterfactual Confidence Bounds  (\texttt{UCCB})}
\label{alg: uccb}

\textbf{Input} tuning parameters $\{\beta_t\}_{t=1}^{\infty}$.

\begin{algorithmic}[1]
\FOR{round $t=1,2,\dots, K$}
\STATE{Choose action $t$.}
\ENDFOR
\FOR{round $t=K+1,K+2,\dots$}
    \STATE{Compute $\widehat{f}_t\in \argmin_{f\in\F}\sum_{i=1}^{t-1}(f(x_i,a_i)-r_i(x_i,a_i))^2$ via the least square oracle.}
    \STATE{Observe $x_t$. }
    \FOR{$i=K+1,K+2,\dots, t$}
    \STATE{Calculate the counterfactual action $\widetilde{a}_{t,i}$ by }
        \begin{align*}
         \widetilde{a}_{t,i}\in\argmax_{a\in\A} \left\{ \widehat{f}_i(x_t,a)+\frac{\beta_i}{\sum_{j=K+1}^{i-1}\mathds{1}\{a=\widetilde{a}_{t,j}\}+1}\right\}.
         \end{align*}
         (ties broken by taking the action with the smallest index)
       \ENDFOR
        \STATE{Take $a_t=\widetilde{a}_{t,t}$ and observe reward $r_t(x_t, a_t)$.}
\ENDFOR
\end{algorithmic}
\end{algorithm}

We present the algorithm that formalizes the high-level descriptions presented in Section \ref{subsec high level}, where $\{\beta_t\}_{t=1}^{\infty}$ are tuning parameters that depends on the statistical complexity of $\F$. With the choice $\beta_t=\sqrt{17t\log(2|\F|t^3/\delta)/K}$ for finite $\F$, the algorithm is simple and achieves  $\tilde{O}(\sqrt{KT\log|\F|})$ regret, which is optimal up to $\log T$ factors.
On the computation side, the algorithm executes no more than $T^2$ maximizations over actions and no more than $T$ calls to the regression oracle.

\begin{theorem}[\bf{Regret for Algorithm \ref{alg: uccb}}]\label{thm uccb}
  Under Assumption \ref{asm realizability} and fixing $ \delta\in (0,1)$, set the parameter $\beta_t$ in Algorithm \ref{alg: uccb} to be $$\beta_t=\sqrt{17t\log(2|\F|t^3/\delta)/K}.$$ Then with probability at least $1-\delta$,  for all $T\geq 1$, the regret of Algorithm $\ref{alg: uccb}$  after $T$ rounds is upper bounded by
\begin{align*}
   \textup{Regret}(T, \textup{Algorithm 1})\leq  2\sqrt{17KT\log(2|\F|T^3/\delta)}(\log(T/K)+1)+\sqrt{2T\log(2/\delta)}+K.
\end{align*}
\end{theorem}
\noindent{\bf Remark: } Recall that our offline regression step can be solved by first-order algorithms and does not require any computation related to the confidence interval (i.e., maintaining a subset of $\F$ or inverting the Hessian). Therefore, despite having much broader applicability, Algorithm \ref{alg: uccb} is also simpler than many  variants of \texttt{UCB} \citep{chu2011contextual, russo2013eluder, foster2018practical} from a computational perspective. The only comparable algorithm to Algorithm \ref{alg: uccb} is  a randomized algorithm---\texttt{FALCON} in \cite{simchi2020bypassing}, which also reduces contextual bandits to offline least square oracle. However,  we believe our optimistic solution should be preferable in many practical settings as our algorithm does not require randomization, has a much simpler analysis, and exhibits much smaller constants in the regret bound.

\subsection{Key ideas underlying \texttt{UCCB}}
We now explain three key ideas underlying \texttt{UCCB}.

\paragraph{Key idea 1: building confidence bounds for policies.}
Previous literature typically refers to the optimism principle as choosing the optimistic action that has the largest estimate on the current context \citep{russo2013eluder, abbasi2011improved, foster2018practical}---optimism is analyzed in the action space. In contrast, we view policies as decisions and build confidence bounds in policy space.
The key step in our approach is to characterize the confidence bounds of the function estimate $\hat{f_t}$, which is the output of the least square oracle given the history $H_{t-1}$. 

For an admissible non-randomized contextual bandit algorithm, at each round $t$ there exists a deterministic stationary policy $\pi_t$ such that the chosen action $a_t$  is equal to $\pi_t(x_t)$ for any realization of $x_t$.  Equivalently, the algorithm selects $\pi_t$ based on $H_{t-1}$ and chooses the action $a_t=\pi_t(x_t)$ at round $t$. Through this viewpoint, the following lemma is applicable to all admissible non-randomized contextual bandit  algorithms:
\begin{lemma}[\bf{confidence of policies}]\label{lemma estimation error regressor}
Consider an admissible non-randomized contextual bandit algorithm that  selects  $\pi_t$ based on $H_{t-1}$ (and chooses the action $a_t=\pi_t(x_t)$) at each round $t$. Then $\forall \delta\in(0,1)$, with probability at least $1-\delta/2$, for all $t>K$ and all $ \pi\in\Pi$, the estimation error on the expected reward of $\pi$ is bounded by
\begin{align}\label{eq: result of lemma 1}
    \big|\E_x[\widehat{f}_t(x,\pi(x))]-\E_x [f^*(x,\pi(x))]\big|\leq \sqrt{\E_x\left[\frac{1}{\sum_{i=1}^{t-1}{\mathds{1}\{\pi(x)=\pi_i(x)\}}}\right]}\sqrt{68\log({2|\F|t^3}/{\delta})}
\end{align}
\end{lemma}
 The proof of Lemma \ref{lemma estimation error regressor} may be interesting in its own right; a proof ketch will be presented in Section \ref{subsec sketch}, and full details are deferred to Appendix \ref{subsec analysis confidence}.
 
\paragraph{Key idea 2: the potential function perspective.}
The idea to establish confidence bounds in policy space  is natural when one takes a potential function perspective. From the potential function perspective,  the cumulative regret of an optimistic algorithm can be approximately bounded by the sum of confidence bounds at all rounds. Therefore, we would like to establish a uniform upper bound whatever the trajectory of policies is, which usually depends on the ``entropy" of the policies.  Although the number of policies is ``large," the ``entropy" of the policies is essentially bounded by $\tilde{O}(K)$ in the following manner. 
\begin{lemma}[\bf{contextual potential lemma}]\label{lemma contextual potential lemma}    Let $\pi_t$ be the policy that chooses action $t$ regardless of $x$ for $t=1,\dots, K$, and from round $K+1 $ up to $T$, its actions are given by any deterministic stationary policy. Then for all $T>K$,
\begin{align*}
    \sum_{t=K+1}^T\E_x\left[\frac{1}{\sum_{j=1}^{t-1}\mathds{1}\{\pi_t(x)=\pi_j(x)\}}\right]\leq K+ K\log(T/K).
\end{align*}
\end{lemma}
The above lemma applies to all admissible non-randomized contextual bandit algorithms that choose each action  once at the first $K$ rounds, regardless of the order by which they are chosen. Proof of this lemma follows from the observation that for every $x\in\X$, the historical sum of $\mathds{1}\{\pi_t(x)=\pi_j(x)\}$ will never exceeds a ``per-context entropy" $O(K\log T)$. In short, analyzing confidence bounds in policy space helps us  take  expectation over the ``per-context entropy," and successfully avoid the dependence on $|\X|$.

\paragraph{Key idea 3: the relaxation tricks and efficient computation.}
Following Lemma \ref{lemma estimation error regressor} and Lemma \ref{lemma contextual potential lemma}, a natural ``upper confidence bound" strategy  is to choose the policy that maximizes the following (unrelaxed) upper confidence bound:
\begin{align*}
    \pi_t\in\argmax_{\pi\in\Pi_{\F}}\left\{ \E_x[\widehat{f}_t(x, \pi(x)]+\sqrt{\E_x\left[\frac{1}{\sum_{i=1}^{t-1}{\mathds{1}\{\pi(x)=\pi_i(x)\}}}\right]}\sqrt{68\log({2|\F|t^3}/{\delta})}\right\},
\end{align*}
where $\Pi_{\F}$ is the policy class defined by $\Pi_{\F}=\{\pi_f:\pi_f(x)\in\argmax_{a\in\A} f(x,a), \forall x\in\X\}$, which contains $\pi_{f^*}$. While we can prove  this strategy leads to optimal regret bounds, it is not directly feasible: 1) the distribution $\D_{\X}$ is unknown; and 2) the optimization over policies is computationally intractable. To solve this issue, we introduce two relaxations:  we ``agnostically" optimize over the full policy space $\Pi$ rather than $\Pi_{\F}$; and we use a simple inequality to relax the confidence bound proved in Lemma \ref{lemma estimation error regressor}, which we call the ``square trick".  

\begin{lemma}[\bf{the ``square trick" relaxation}]\label{lemma square trick}
The inequality \eqref{eq: result of lemma 1} can be further relaxed to
\begin{align}\label{eq: result of lemma 3}
    \big|\E_x[\widehat{f}_t(x,\pi(x))]-\E_x[f^*(x,\pi(x))]\big|\leq {\E_x\left[\frac{\beta_t}{\sum_{i=1}^{t-1}{\mathds{1}\{\pi(x)=\pi_i(x)\}}}\right]}+ \frac{K\beta_t}{t}.
\end{align}
\end{lemma}
\begin{proof}
Simply relax \eqref{eq: result of lemma 1} by the Arithmetic Mean-Geometric Mean  inequality.
\end{proof}

By performing the two relaxations stated above, we only need to consider the optimization problem
\begin{align}\label{eq: policy maximization}
    \pi_t\in\argmax_{\pi\in\Pi} \left\{\E_x[\widehat{f}(x, \pi(x)]+{\E_x\left[\frac{\beta_t}{\sum_{i=1}^{t-1}{\mathds{1}\{\pi(x)=\pi_i(x)\}}}\right]}\right\}.
\end{align}
This is a ``per context" optimization problem, where  optimality at every context implies optimality of $\pi_t$ over the full policy space $\Pi$. The algorithm does not need to calculate $\pi_t$ explicitly in every step. Instead, the algorithm observes $x_t$, and calculates all the counterfactual actions  $\pi_1(x_t),\pi_2(x_t),\dots, \pi_{t-1}(x_t)$ as if the past policies were applied at $x_t$. Using these counterfactual actions, the algorithm calculates a counterfactual confidence, and chooses an optimistic action $a_t$ that maximize the upper   confidence bound stated in \eqref{eq: policy maximization}.

The formula to calculate the counterfactual action $\pi_i(x_t)$, 
\begin{align*}
      \pi_i(x_t)\in\argmax_{a\in\A} \left\{ \widehat{f}_i(x_t,a)+\frac{\beta_i}{\sum_{j=1}^{i-1}\mathds{1}\{a=\pi_{j}(x_t)\}}\right\},
\end{align*}
requires us to the compute the sequence $\{\widetilde{a}_{t,i}\}_{i=1}^t$ in a recursive manner: for $i=1,\dots, t$, compute
\begin{align*}
         \widetilde{a}_{t,i}\in\argmax_{a\in\A} \left\{ \widehat{f}_i(x_t,a)+\frac{\beta_i}{\sum_{j=K+1}^{i-1}\mathds{1}\{a=\widetilde{a}_{t,j}\}+1}\right\}.
         \end{align*}
And finally we take $a_t=\pi_t(x_t)=\widetilde{a}_{t,t}$. Therefore, we can explain the explicit steps in Algorithm \ref{alg: uccb} via the following (obvious) equivalence:
\begin{lemma}[\bf{equivalence between  Algorithm \ref{alg: uccb} and implicit strategy \eqref{eq: policy maximization}}]\label{lemma equivalence}
  After the first $K$ initialization rounds, Algorithm \ref{alg: uccb} produce the same pathwise actions as those produced by the policies $\{\pi_t\}_{t>K}$ chosen by the upper-confidence-bound rule \eqref{eq: policy maximization} and a specific tie-breaking rule (i.e., 
when there are multiple solutions to \eqref{eq: policy maximization}, taking $\pi_t$ to be the unique solution such that for all other solutions $\pi'$ to \eqref{eq: policy maximization} and all $x\in\X$, the index of the action $\pi_t(x)$ is smaller than the index of the action $\pi'(x)$).
\end{lemma}
Based on all the lemmas that we introduce in this subsection, one can prove the $\widetilde{O}(\sqrt{KT\log|\F|})$ regret bound for Algorithm \ref{alg: uccb} through relatively standard techniques. The full proof is deferred to  Appendix \ref{sec proof finite action}, and a sketch is provided below.

\subsection{Proof sketch of Theorem \ref{thm uccb} and Lemma \ref{lemma estimation error regressor}}\label{subsec sketch}

In this subsection we present a proof sketch of Theorem \ref{thm uccb} (the cumulative regret of Algorithm \ref{alg: uccb}) and Lemma \ref{lemma estimation error regressor} (confidence bounds in policy space, whose relaxation leads to Lemma \ref{lemma square trick}).

\paragraph{\bf Proof sketch of Theorem \ref{thm uccb}.} From Lemma \ref{lemma equivalence}, we know Algorithm \ref{alg: uccb} implicitly chooses the optimistic policy $\pi_t$ (i.e., solution of \eqref{eq: policy maximization}) at each round $t$. We prove the regret bound on the event where the inequality  \eqref{eq: result of lemma 3} holds true for all $\pi\in\Pi$. From Lemma \ref{lemma square trick}, the measure of this event is at least $1-\frac{\delta}{2}$. 

 Optimism of Algorithm \ref{alg: uccb} in policy space suggests that for all $t>K$,
\begin{align*}
    \E_x[f^*(x,\pi_{f^*}(x)]\leq \E_x[\widehat{f}_t(x, \pi_{f^*}(x))]+{\E_x\big[\frac{\beta_t}{\sum_{i=1}^t{\mathds{1}\{\pi_{f^*}(x)=\pi_i(x)\}}}\big]}+\frac{K\beta_t}{t}\nonumber\\
    \leq \argmax_{\pi\in\Pi}\bigg\{\E_x[\widehat{f}_t(x, \pi(x))]+{\E_x\big[\frac{\beta_t}{\sum_{i=1}^t{\mathds{1}\{\pi(x)=\pi_i(x)\}}}\big]}\bigg\}+\frac{K\beta_t}{t}\nonumber\\
    =\E_x[\widehat{f}_t(x, \pi_{t}(x))]+{\E_x\big[\frac{\beta_t}{\sum_{i=1}^t{\mathds{1}\{\pi_{t}(x)=\pi_i(x)\}}}\big]}+\frac{K\beta_t}{t}\nonumber\\
    \leq \E_x[f^*(x,\pi_t(x))]+{\E_x\big[\frac{2\beta_t}{\sum_{i=1}^t{\mathds{1}\{\pi_{t}(x)=\pi_i(x)\}}}\big]}+\frac{2K\beta_t}{t},
\end{align*}
where the first and the last inequality are due to Lemma \ref{lemma square trick}; and the second inequality due to maximization over policies. Therefore, the expected regret incurred at round $t$ is bounded by 
\begin{align}\label{eq: result of standard analysis optimism}
    \E_x[f^*(x,\pi_{f^*}(x)]-\E_x[f^*(x,\pi_t(x)]\leq {\E_x\left[\frac{2\beta_t}{\sum_{i=1}^{t-1}{\mathds{1}\{\pi_{t}(x)=\pi_i(x)\}}}\right]}+\frac{2K\beta_t}{t}.
\end{align}
 Taking the telescoping sum of \eqref{eq: result of standard analysis optimism} and applying the contextual potential lemma (Lemma \ref{lemma contextual potential lemma}), we can prove
 \begin{align}\label{eq: can prove in sketch}\sum_{t=1}^T \E_x[f^*(x,\pi_{f^*}(x)-f^*(x,\pi_t(x))]\leq 2\sqrt{17KT\log(2|\F|T^3/\delta)}(\log(T/K)+1)+K.\end{align}
By Azuma's inequality and Lemma \ref{lemma equivalence}, with probability at least $1-\delta/2$, we can bound the regret by 
\begin{align}\label{eq: second last in sketch}
    \text{Regret}(T, \text{Algorithm \ref{alg: uccb}})\leq  \E_x[f^*(x,\pi_{f^*}(x)-f^*(x,\pi_t(x))]+\sqrt{2T\log(2/\delta)}.
\end{align}
Finally we combine \eqref{eq: can prove in sketch} and \eqref{eq: second last in sketch} by a union bound to finish the proof.

\paragraph{Proof sketch of Lemma \ref{lemma estimation error regressor}.} The proof of Lemma \ref{lemma estimation error regressor} includes three key steps: characterization of the estimation error  (inequality \eqref{eq: uniform over sequence}); a counting argument (inequality \eqref{eq: before cauchy main text}); and applying Cauchy-Schwartz inequality to \eqref{eq: before cauchy main text}. Now we describe these key steps.

 The following lemma, which holds for arbitrary algorithms, characterizes the estimation errors of an arbitrary sequence of estimators.
\begin{lemma}[uniform convergence over all sequences of estimators]\label{lemma uniform convergence time}  For an arbitrary contextual bandit algorithm, $\forall \delta\in (0,1)$, with probability at least $1-{\delta}/2$, 
\begin{align}\label{eq: uniform over sequence}
\sum_{i=1}^{t-1}\E_{x_i,a_i}\left[(f_{t}(x_i,a_i)-f^*(x_i,a_i))^2|H_{i-1}\right]\le  68\log({2|\F|t^3}/{\delta}) \nonumber\\+2\sum_{i=1}^{t-1}(f_t(x_i,a_i)-r_i(x_i,a_i))^2-(f^*(x_i,a_i)-r_i(x_i,a_i))^2,
\end{align}
uniformly over all $t\geq 2$ and all fixed sequence $f_2, f_3, \dots\in\F$.
\end{lemma}
Proof of Lemma \eqref{eq: uniform over sequence} can be found in Appendix \ref{subsec analysis confidence}.

Consider the contextual bandit algorithm that choose $\pi_t$ based on $H_{t-1}$  at each round $t$, the left hand side of \eqref{eq: uniform over sequence} is equal to $\sum_{i=1}^{{t-1}}\E_{x}\left[(f(x,\pi_i(x))-f^*(x,\pi_i(x)))^2\right]$. Then by using the fact that $\forall \pi\in\Pi$, for all $x\in\X$, $$\mathds{1}\{\pi(x)=\pi_i(x)\}(f_t(x,\pi(x))-f^*(x_i,\pi(x)))^2\leq (f_t(x,\pi_i(x))-f^*(x,\pi_i(x)))^2,$$
we obtain the key inequality \begin{align}\label{eq: before cauchy main text}\E_x \Big[\sum_{i=1}^{t-1} \mathds{1}\{\pi(x)=\pi_i(x)\}(f_t(x,\pi(x))-f^*(x_i,\pi(x)))^2\Big]\leq  68\log({2|\F|t^3}/{\delta}) \nonumber\\+2\sum_{i=1}^{t-1}(f_t(x_i,a_i)-r_i(x_i,a_i))^2-(f^*(x_i,a_i)-r_i(x_i,a_i))^2.\end{align}
We then apply Cauchy-Schwartz inequality to lower bound the left hand side of \eqref{eq: before cauchy main text}, and take $f_t=\widehat{f}_t$ be the least square solutions to upper bound the right hand side of \eqref{eq: before cauchy main text}.

\subsection{Generalization to infinite $\F$}\label{subsec infinite function class}

Extensions of our theory to ``infinite" $\F$ with  statistical complexity notions of covering number and parametric dimension are straightforward. Technically speaking, we only  require some standard uniform convergence arguments to modify Lemma \ref{lemma uniform convergence time}. We will first show that our results trivially generalizes to parametric $\F$ with suitable continuity, and then  extend our results to general function  classes following some more careful covering arguments.
 
\paragraph{\bf Parametric dimension.} Assume $\F$ is parametrized by a compact set  $\Theta\subset \R^d$ whose diameter is bounded by $\Delta$, and satisfies
\begin{align}\label{eq: lipchitz continuous}
    |f_{\theta_1}(x,a)-f_{\theta_2}(x,a)|\leq L\|\theta_1-\theta_2\|,
\end{align}
uniformly over $x\in\X$ and $a\in\A$. 
This case clearly covers many previous structured models (variants of the  ``linear payoff" formulation \eqref{eq: linear contextual bandit}). 

\begin{corollary}[\bf{extension to infinite $\F$ via parametric dimension}]\label{coro parametric}
Under Assumption \ref{asm realizability} and the assumption \eqref{eq: lipchitz continuous} and fixing $\delta\in (0,1)$, set the parameter $\beta_t$ in Algorithm \ref{alg: uccb} to be $$\beta_t=\sqrt{34t/K}\sqrt{d\log(2+{\Delta Lt})+\log({2t^3}/{\delta})+1}.$$ Then Algorithm \ref{alg: uccb} satisfies that with probability at least $1-\delta$,  for all $T\geq 1$, $$\textup{Regret}(T, \textup{Algorithm 1})\leq 2K\beta_T (\log(T/K)+1)+\sqrt{2T\log(2/\delta)}+K=\tilde{O}(\sqrt{KTd}).$$
\end{corollary}
\noindent{\bf Remark: } While this regret bound has a worse dependence on $K$  in the  ``linear payoff" formulation \eqref{eq: linear contextual bandit} compared with  \texttt{SupLinUCB} in \cite{chu2011contextual} (whose regret is logarithmic in $K$),  Algorithm \ref{alg: uccb} can be applied in more general parametric settings and enjoys much lower computational demands (there is no need to invert any Hessian). While the square-root dependence on $K$ can not be improved for general $\F$  (see the lower bound in \cite{agarwal2012contextual}), we can improve this dependence for structured models by applying our  results in Section \ref{sec infinite}.

\paragraph{\bf Covering number formulation.} Our results can be extended to general (possibly non-parametric) function classes via covering numbers and standard uniform convergence techniques.  We consider formulation \eqref{eq: broad model}---a major target of previous works on general contextual bandits \citep{krishnamurthy2019active, foster2018practical, foster2020beyond}. We assume access to a general function class $\G$ that contains mappings from $\X$ to $[0,1]$, and assume
\begin{align}\label{eq: broad model}
\F=\{f: f(x,a)=g_a(x), \quad  g_a\in \mathcal{G}\}.
\end{align}
\begin{definition}[\bf{covering number}] For a function class $\G$ that contains mappings from $\X$ to $[0,1]$ and fixed $n\in{\mathbb{Z}_{+}}$, an empirical $L_1$ cover on a sequence $x_1,\dots, x_n$ at scale $\eps$ is a set $U\subseteq \R^n$ such that 
$
    \forall g\in\mathcal{G}, \exists u\in U, \frac{1}{n}\sum_{i=1}^n|g(x_n)-u_n|\leq \eps.
$
We define the covering number $\mathcal{N}_1(\mathcal{G},\eps,\{x_i\}_{i=1}^{n})$ to be the size of the smallest such cover.
\end{definition}
 Given careful covering arguments proved in \cite{krishnamurthy2019active, foster2018practical}, the following extension  is  straightforward:
\begin{corollary}[\bf{extension to infinite $\F$ via covering number}]\label{coro general class}
Under Assumption \ref{asm realizability} and the assumption \eqref{eq: broad model}, given $T\geq 1$ and  $\delta\in (0,1)$, by setting all the parameters $\beta_t$ in Algorithm \ref{alg: uccb} to be a fixed value 
$$\beta=\sqrt{TK}\cdot\inf_{\eps>0}\left\{25\eps T+80\log\left(\frac{8K T^3 \E_{\{x_i\}_{i=1}^{T}}\mathcal{N}_1(\G, \eps, \{x_i\}_{i=1}^{T})}{\delta}\right)\right\}.$$
Then, Algorithm \ref{alg: uccb} satisfies that with probability at least $1-\delta$, 
$$
  \textup{Regret}(T, \textup{Algorithm 1})\leq   2K\beta(\log(T/K)+1)+\sqrt{2T\log(2/\delta)}+K.
$$
\end{corollary}

\section{A unified framework for infinite-action spaces}\label{sec infinite}
In this section we study infinite-action contextual bandits to illustrate the simplicity and applicability of the \texttt{UCCB} principle. In context-free settings, discussion on infinite actions can be sorted into two streams. The first stream studies variants of the linear action model. Prominent examples include linearly parametrized bandit \citep{dani2008stochastic, abbasi2011improved}, and parametrized bandit with generalized linear model \citep{filippi2010parametric}. The second stream is based on discretization  over actions and reduction to the finite-action setting (e.g.,  Lipchitz bandit \citep{kleinberg2005nearly}). We focus on the first stream here, as it exhibits additional challenges of efficient exploration beyond the finite-action setting.

To focus on the core messages, we assume $\F$ to be finite and function in $\F$ take values in $[0,1]$.
We propose a generic algorithm (Algorithm \ref{alg: uccb ia}) that achieves 
\begin{align*}
    \textup{Regret}(T, \textup{Algorithm 2})\leq   \tilde{O}(\sqrt{\mathcal{E}\log|\F|T}),
\end{align*} for many models of interest. Here  we call $\mathcal{E}:=\E_x[\mathcal{E}_x]$ the ``average decision entropy," where $\mathcal{E}_x$ is (informally) the complexity of the ``fixed-$x$-model" where the context is fixed to be $x$. Note that unlike previous complexity measures such as  ``Eluder dimension" \citep{russo2013eluder}, the ``average decision entropy" $\mathcal{E}$ does not scale with $|\X|$ so that this complexity measure is much more useful in the contextual settings. 
We will present several interesting illustrative examples, and present key ideas of our algorithm using these examples.

\subsection{Illustrative models}
In the context-free infinite-action bandits literature, it is well-known that $\tilde{O}(\sqrt{T})-$type regret is only possible for structured models, among which variants of linear bandits are the  preponderant models. As a result,  our framework mainly targets settings where all ``fixed-$x$-model" are variants of linear bandits. 
\begin{example}[\bf{contextual bandit with linear action model}]\label{example cb linear} Given a general vector-valued function class $\G$ that contains mappings from $\X$ to $\R^d$, let
\begin{align}\label{eq: linear semi}
    \F=\{f: \exists g\in \G {\textup{ s.t. }} f(x,a)=g(x)^\top a, \forall x\in\X,\forall a\in\A\}.
\end{align}
We assume $\A\subset\R^d$ is an arbitrary compact set, and is available for the agent at all rounds. This formulation is a strict generalization of the finite-action realizable contextual bandit problem we studied in previous sections (it reduces to the $K-$armed setting when $\A$ is the set of $K$ element vectors in $\R^K$).
The concurrent work \cite{foster2020adapting} also studies this formulation, but they assume access to the stronger online regression oracle which is not computationally efficient in general (in contrast to our use of the weaker and more practical offline regression oracle).  Formulation \eqref{eq: linear semi} was also studied  in  \cite{chernozhukov2019semi} but the goal there was off-policy evaluation rather than regret minimization. 

 With knowledge on linear bandits we can prove $\mathcal{E}_x=d$ for all $x\in\X$. (detailed explanation is deferred to  Section \ref{subsubsec illutration linear}). Therefore $\mathcal{E}=d$, which is independent of the number of actions, and the order of regret is expected to be  $\tilde{O}(\sqrt{d\log|\F|T})$.
\end{example}

\begin{example}[\bf{contextual bandit with generalized linear action model.}]\label{example cb glm}
Consider a broader choice of models, which contains generalized linear action models and allows a mapping $\varphi$:
\begin{align}\label{eq: cb glm}
   \F=\left\{f: \exists g\in\G \textup{ s.t. }f(x,a)=\sigma_x\left(g(x)^\top\varphi(x, a)\right),  \forall x\in\X, \forall a\in\A\right\},
\end{align} where for every $x\in\X$,  $\sigma_x: \R\rightarrow [0,1]$ is a known link function that satisfies
\begin{align*}
    \frac{\sup_{a}\sigma_x'(\langle g^*(x), \varphi(x,a)\rangle)}{\inf_{a}\sigma_x'(\langle g^*(x), \varphi(x,a)\rangle)}\leq \kappa_x;
\end{align*}
and $\varphi: \X\times\A\rightarrow \R^d$ is a known compactness-preserving mapping (e.g., continuous mappings). This model generalizes \eqref{eq: linear semi} and allows more flexibility.  When we set $\varphi(x,a)=x_a$, we see that this model is significantly broader in scope than the simple ``linear payoff" formulation \eqref{eq: linear contextual bandit}, as $g(x)$ is a general function that depends on $x$ rather than a fixed parameter $\theta$.

Our analysis will show that $\mathcal{E}_x=\kappa_x^2d$ for all $x\in\X$ (detailed explanation is deferred to Section \ref{subsubsec illustration generalized linear}),  so that $\mathcal{E}=\E_x[\kappa_x^2]d$, and the order of regret is expected to be $\tilde{O}(\sqrt{E_x[\kappa_x^2]d\log|\F|T})$.
\end{example}

\begin{example}[\bf{heterogeneous action set}]\label{example cb heterogeneous}
 Many real-world, customized pricing and personalized healthcare applications have a high dimensional action set $\A$, but the ``effective dimension" of available actions after observing $x$ is usually much smaller. To model these applications, consider  the reward model
\begin{align}\label{eq: heterogeneous}
    \F=\left\{f:\exists g\in\G \textup{ s.t. }f(x,a)=\sigma_x(g(x)^\top a),  \forall x\in\X, \forall a\in\A(x)\right\},
\end{align}
where for all $x\in\X$ we assume a  compact action set $\A(x)\subset\A$, and assume $\A(x)$ is contained in a  $d_x-$dimensional subspace. When the agent observes context $x$, she can only choose her action from $\A(x)$. 

For this model we have $\mathcal{E}_x=\kappa_x^2d_x$ (detailed explanation is deferred to Section \ref{subsubsec illustration hetergeneous}) so that $\mathcal{E}=\E_x[\kappa_x^2d_x]$. The salient point here is the we avoid dependence on the full dimension $d$. Regret therefore scales as $\tilde{O}(\sqrt{\E_x[\kappa_x^2d_x]\log|\F|T})$.
\end{example}

\subsection{Counterfactual action divergence}
 The main modification required for infinite-action settings is predicated on a central concept called ``counterfactual action divergence," which generalizes the term $({\sum_{i=1}^n\mathds{1}\{a=a_i\}})^{-1}$ that was used in Algorithm \ref{alg: uccb}. This new concept characterizes ``how much information" is learned from action $a$ given a sequence $\{a_i\}_{i=1}^{n}$, on  the ``fixed-$x$-model."

\begin{definition}[\bf{counterfactual action divergence}] For fixed integer $n$, a context $x$, an action $a$ and a  sequence of actions $\{a_i\}_{i=1}^{n}$, we say $V_x(a||\{a_i\}_{i=1}^{n})$ is a proper choice of the counterfactual action divergence between $a$ and  $\{a_i\}_{i=1}^{n}$ evaluated at $x$, if 
\begin{align*}
   V_x(a||\{a_i\}_{i=1}^{n})\geq \sup_{f\in\F}\bigg\{ \frac{|f(x,a)-f^*(x,a)|^2}{ \sum_{i=1}^{n}(f(x,a_i)-f^*(x,a_i))^2}\bigg\}.
 \end{align*}
 We define $V_x(a||\emptyset)=\infty$ in the case $n=1$.
\end{definition}

 Using the definition of counterfactual action divergence,  the expectation  \begin{align}\label{eq: policy divergence}\E_x [V_x(\pi(x)||\{\pi_i(x)\}_{i=1}^{t-1})],\end{align}  can be used to construct an upper confidence bound on the expected reward of policy  $\pi$ given the past chosen policies $\{\pi
\}_{i=1}^{t-1}$. Similar to the finite-action setting, the agent chooses the optimistic policy $\pi_t$ that maximizes this confidence bound, and  chooses $a_t=\pi_t(x_t)$ without explicitly computing $\pi_t$---this is achieved by sequentially recovering counterfactual actions, as will be illustrated in our proposed Algorithm.

Convenient choices of $V_x(a||\{a_i\}_{i=1}^{n}\}_{i=1}^{n})$ should be taken  case by case for different problems. In the following lemma, we present closed-form choices of $V_x(a||\{a_i\}_{i=1}^{n}\}_{i=1}^{n})$ in all our illustrative examples.
\begin{statement}[{\bf illustration of counterfactual action divergences}]\label{statement counterfactual action divergence}In the illustrative examples, the counterfactual action divergences are given as follows (and taken as $\infty$ when inverse of matrices is not well-defined):
\begin{itemize}
\item finite-action contextual bandit: 
\begin{align*}
 V_x(a||\{a_i(x)\}_{i=1}^{n})=\frac{1}{\sum_{i=1}^{n}\mathds{1}\{a=a_i\}}.
\end{align*}

\item linear action model \eqref{eq: linear semi}: 
\begin{align}\label{eq: conterfactual action divergence linear coro}
V_x(a||\{a_i\}_{i=1}^{n}\}_{i=1}^{n})=a^\top(\sum_{i=1}^{n}[a_ia_i^\top])^{-1}a.
\end{align}

\item generalized linear action model \eqref{eq: cb glm}:  
\begin{align*}
V_x(a||\{a_i\}_{i=1}^{n})=\kappa_x^2\varphi(x,a)^\top(\sum_{i=1}^{n}[\varphi(x,a_i)\varphi(x,a_i)^\top])^{-1}\varphi(x,a).
\end{align*}

\item generalized linear action model with heterogeneous action sets \eqref{eq: heterogeneous}:
\begin{align*}
V_x(a||\{a_i(x)\}_{i=1}^{n})=\kappa_x^2b_{x,a}\top(\sum_{i=1}^{n}[b_{x,a_i}b_{x,a_i}^\top])^{-1}b_{x,a},
\end{align*}
where $b_{x,a}$ is the coefficient vector of $a$ with a  basis  $\{A_{x,1}, \dots, A_{x,d_x}\}$ of $\A(x)$, i.e.,
\begin{align*}
   a= [A_{x,1}, \dots, A_{x,d_x}]b_{x,a}.
\end{align*}
\end{itemize}
\end{statement}

\subsection{The algorithm and  regret bound}
Algorithm \ref{alg: uccb ia} is a  generalization of Algorithm \ref{alg: uccb} to the infinite-action setting. It can be applied to most parametric action models that have been studied in the context-free setting, and handles  heterogeneous action sets. Recall that $\mathcal{E}:=\E[\mathcal{E}_{x}]$ is the average decision entropy of the problem, for which we will give the formal definition later. The ``initialization oracle" and the ``action maximization oracle" will also be explained shortly.

\begin{algorithm}[htbp]
\caption{ Upper Counterfactual Confidence Bound-Infinite Actions  (\texttt{UCCB-IA})}
\label{alg: uccb ia}
{\textbf{Input} tuning parameters $\{\beta_t\}_{t=1}^{\infty}$.}
\begin{algorithmic}[1]
\FOR{round $t=1,2,\dots$}
    \STATE{Compute $\widehat{f}_t=\arg\min_{f\in\F}\sum_{t=1}^{t-1}(f(x_t,a_t)-r_t(x_i,a_t))^2$ via the least square oracle.}
    \STATE{Observe $x_t$, use the initialization oracle to obtain initializations $\{A_{x_t,i}\}_{i=1}^{d_x}$.}
    \FOR {$i=1,2,\dots, t\lor d_x$}
    \STATE{Take $\widetilde{a}_{t,i}=A_{x_t, i}$.}
\ENDFOR
  \FOR {$i=t\land (d_x+1), \dots ,t$}
  \STATE{Use the action maximization oracle to compute  counterfactual actions: 
        $$
         \widetilde{a}_{t,i}\in\argmax_{a\in\A(x_t)} \left\{ \widehat{f}_i(x_t,a)+{\beta_i}V_{x_t}(a||\{\widetilde{a}_{t,j}\}_{j=1}^{i-1})\right\}.
        $$}
     \ENDFOR
        \STATE{Take $a_t=\widetilde{a}_{t,t}$ and observe reward $r_t(x_t, a_t)$.}
\ENDFOR
\end{algorithmic}
\end{algorithm}

Algorithm \ref{alg: uccb ia} essentically provide a reduction from contextual models to the ``fixed-$x$-models."  The regret of an optimistic algorithm is usually upper bounded by the sum of confidence bounds. In our case, the sum of expectations \eqref{eq: policy divergence} is decomposable over contexts, so tractability of the ``fixed-$x$-models" suffices to make Algorithm \ref{alg: uccb ia} provably efficient. 
Formally, we require regularity conditions so that the ``fixed-$x$-models" are solvable by the optimism principle. Motivated by the standard potential arguments used in the linear bandit literature, we make Assumption \ref{asm per context} below. Verification of this assumption on Examples \ref{example cb linear}-\ref{example cb heterogeneous} will be presented in the next section.
\begin{assumption}[{\bf per-context models  are solvable by optimism}]\label{asm per context}
There exists counterfactual action divergences  such that   the following are satisfied:

i)for all $x\in\X$,
 there exists  $d_x$ actions $A_{x,1}, \dots, A_{x,d_x}\in\A(x)$ such that  $V_x(a||\{A_{x,i}\}_{i=1}^{d_x})<\infty$ for all $a\in\A(x)$.

ii) For all $x\in\X$, there exists $\mathcal{E}_x>0$ such that for all $T\geq 1$ and  all sequences $\{a_t\}_{t=1}^{T}$ that satisfy $\{a_t\}_{t=1}^{d_x\land T}=\{A_{x,t}\}_{t=1}^{d_x\land T}$, we have
\begin{align*}
    \sum_{t=1}^T  \left[1\land V_x(a_t||\{a_j\}_{j=1}^{t-1})\right]\leq \mathcal{E}_x\textup{poly}(\log T)
\end{align*}
for all $x\in\X$, where $\textup{poly}(\cdot)$ is a fixed polynomial-scale function.
\end{assumption}

Given positive values $\mathcal{E}_x$ that satisfies  condition ii) in Assumption \ref{asm per context}, we  define  $\mathcal{E}_x:=\E_x[\mathcal{E}_x]$ to be (a proper choice of) the ``average decision entropy" of the problem. The ``average decision entropy" of a problem is not unique, and any ``proper" choice of $\mathcal{E}$ leads to a rigorous regret bound of Algorithm \ref{alg: uccb ia}.

Besides the least-square oracle, Algorithm \ref{alg: uccb ia} uses two other optimization oracles that are necessary in the infinite-action setting: 1) a 
deterministic initialization oracle which returns  $\{A_{x,i}\}_{i=1}^{d_x}$
satisfying Assumption \ref{asm per context} after inputting $\A(x)$ (this is standard for Examples \ref{example cb linear}-\ref{example cb heterogeneous} using the theory of barycentric spanners, see the next subsection); and 2)  a deterministic action maximization oracle whose output is a maximizer of a function over the feasible region $\A(x)$.

After imposing the regularity conditions proposed in Assumption \ref{asm per context}, the regret of Algorithm \ref{alg: uccb ia} can be bounded as the follows.

\begin{theorem}[\bf{Regret of Algorithm \ref{alg: uccb ia}}]\label{thm uccb infinite}
Under Assumptions \ref{asm realizability} and \ref{asm per context} and fixing $\delta\in (0,1)$, let $$\beta_t=\sqrt{17t\log(2|\F|t^3/\delta)/\mathcal{E}}.$$ Then with probability at least $1-\delta$, for all $ T\geq 1$ the regret of Algorithm $\ref{alg: uccb ia}$  after $T$ rounds is upper bounded by
\begin{align*}
  \textup{Regret}(T, \textup{Algorithm 2})\leq 2\sqrt{17\mathcal{E}T\log(2|\F|T^3/\delta)}(\textup{poly}(\log T)+1)+\sqrt{2T\log(2/\delta)}+\mathcal{E}.
\end{align*}
\end{theorem}
This theorem immediately provides regret bounds for all our illustrative examples, which we will discuss in the next subsection.

Finally, we give a high-level interpretation of the average decision entropy $\mathcal{E}$: if the expectation \eqref{eq: policy divergence} is the ``discrete" partial gradient of a potential function, then the historical sum has  the path independence property---that is,  the historical sum of \eqref{eq: policy divergence} can be bounded by the maximum value of a potential function, which is characterized by the average decision entropy  $\mathcal{E}$. Since $\mathcal{E}$ is the average rather than the sum of the effective complexities of all ``fixed-$x$-models," \texttt{UCCB} provides a generic solution to achieve optimal regret bounds that do not scale with  $|\X|$.

\subsection{Applications in illustrative examples}
In this subsection we will carefully go through the three  illustrative examples. We summarize the conclusions in the following corollary:
\begin{corollary}[{\bf Theorem \ref{thm uccb infinite} applied to illustrative examples}]\label{corollary infinite}
Examples  \ref{example cb linear}-\ref{example cb heterogeneous} satisfy Assumptions \ref{asm per context} with the average decision entropy given by
\begin{itemize}
    \item  linear action model \eqref{eq: linear semi}: $\mathcal{E}= d$.
    \item  generalized linear action model \eqref{eq: cb glm}: $\mathcal{E}=\E_x[\kappa_x^2] d$.
    \item generalized linear action model with heterogeneous action sets \eqref{eq: heterogeneous}: $\mathcal{E}=\E_x[\kappa_x^2 d_x]$. 
\end{itemize}
\end{corollary}
Now we give a verification in the remaining parts of this subsection.

\subsubsection{Contextual bandits with linear action model (Example \ref{example cb linear}).} \label{subsubsec illutration linear}
We begin with contextual bandits with linear action model \eqref{eq: linear semi}, with the  homogeneous action set $\A$. For this problem, Algorithm \ref{alg: uccb} only needs to compute the initialization actions $A_1, \dots, A_d$ once, and use them during the first $d$ rounds. This suffices to complete the required initialization for all contexts.

Based on  well-known results in the  linear bandit literature, it is straightforward to show that 
$\mathcal{E}= d $, because we can take $\mathcal{E}_x= d$ for every per-context model.
The details are as follows. 

As shown in Statement \ref{statement counterfactual action divergence}, for all $x\in\X$, we choose the counterfactual action divergence between any $a_t$ and any sequence $\{a_i\}_{i=1}^{t-1}$ evaluated at $x$ to be
\begin{align*}
  V_x(a_t||\{a_i\}_{i=1}^{t-1})= a_t^\top(\sum_{i=1}^{t-1}[a_i a_i^\top])^{-1}a_t.
\end{align*}
 Following the standard approach in the linear bandit literature (e.g., see \cite{dani2008stochastic}),  we choose the $d$ initialization actions $\{A_{i}\}_{i=1}^{d}$ to be the barycentric spanner of $\A$. A barycentric spanner is a set of $d$ vectors, all contained in $\A$, such that every vector in $\A$ can be expressed as a linear combination of the spanner with coefficients in $[-1,1]$.  An efficient algorithm to find the barycentric spanner for an arbitrary compact set is given in \cite{awerbuch2004adaptive}.

The following result follows Lemma 9 in \cite{dani2008stochastic}\footnote{Lemma 9 in \cite{dani2008stochastic} holds for an arbitrary compact set $\A\subset\R^d$, as changing the coordinate system is without the loss of generality for this lemma.}, which is often referred to as the ``elliptical potential lemma": let $a_i=A_i$ for $i=1,\dots, d$, then for all $T>d$ and all trajectory $\{a_t\}_{t=d+1}^{T}$,
\begin{align*}
    \sum_{t=d+1}^T \left[1\land a_t^\top(\sum_{i=1}^{t-1}[a_i a_i^\top])^{-1}a_t\right]\leq 2d \log T.
\end{align*}
Therefore, we obtain for all $T\geq 1$ and all $x\in\X$,
\begin{align}\label{eq: potential linear case}
    \sum_{t=1}^T \left[1\land V_x(a_t||\{a_i\}_{i=1}^{t-1})\right]\leq 2d \log T+d\leq 3d\log T.
\end{align}
By taking $\{A_i\}_{i=1}^d$ to be a barycentric spanner of $\A$, setting $\mathcal{E}=d$, and taking $\textup{poly}(\log T)=3\log T$, Assumption \ref{asm per context} holds for problem \eqref{eq: linear semi}. Despite the illustration here, we also note that our Assumption \ref{asm per context} is not restricted to any particular choice of initialization actions and $\mathcal{E}$: there are other ways to choose linearly independent initialization actions, giving rise to a slightly different  $\textup{poly}(\log T)$ term in Assumption \ref{asm per context}  (see, e.g., Lemma 11 in \cite{abbasi2011improved}).

   \subsubsection{Contextual bandits with generalized linear action model (Example \ref{example cb glm}).} \label{subsubsec illustration generalized linear}
   For the problem formulation \eqref{eq: cb glm}, we can take $\E[\mathcal{E}_x]=\E[\kappa_x^2]d$. The details are as follows.
   
   As shown in Statement \ref{statement counterfactual action divergence}, given $x\in\X$, we choose the counterfactual action divergence between any $a_t$ and any sequence $\{a_i\}_{i=1}^{t-1}$ evaluated at $x$ to be
\begin{align*}
  V_x(a_t||\{a_i\}_{i=1}^{t-1})= \kappa_x^2 \varphi(x,a_t)^\top(\sum_{i=1}^{t-1}[\varphi(x,a_i)\varphi(x, a_i)^\top])^{-1}\varphi(x,a_t).
\end{align*} Given $x\in\X$, we take $\{A_{x,i}\}_{i=1}^d$ such that $\{\varphi(x,A_{x,i})\}_{i=1}^d$ consists of a barycentric spanner of $\{\varphi(x,a):a\in\A\}$\footnote{in  formulation \eqref{eq: cb glm} we have asked  $\varphi$ to preserve compactness  with respect to $a$ (e.g. the continuous ones), so such barycentric spanner must exists.}. Note that a different basis $\{A_{x,i}\}_{i=1}^d$ should be computed  for each $x$.  From our previous result \eqref{eq: potential linear case} and the fact $\kappa_x\geq 1$, for all $T\geq 1$ and all sequences $\{a_i\}_{i=1}^{T}$ that satisfy $\{a_i\}_{i=1}^{d_x\land T}=\{A_{x,i}\}_{i=1}^{d_x\land T}$,
   \begin{align*}
       \sum_{t=1}^T \left[1\land V_x(a_t||\{a_i\}_{i=1}^{t-1})\right]=\sum_{t=1}^T \left[1\land \kappa_x^2 \varphi(x,a)^\top(\sum_{i=1}^{t-1}[\varphi(x,a_i) \varphi(x,a_i)^\top])^{-1}\varphi(x,a)\right] \leq \kappa_x^2 3d\log T.
   \end{align*}
    By taking
   $\mathcal{E}_x=\kappa_x^2 d$, and  $\textup{poly}(\log T)=3\log T$, Assumption \ref{asm per context} holds with $\mathcal{E}=\E_x[\kappa_x^2]d$.

\subsubsection{Contextual bandits with heterogeneous action set (Example \ref{example cb heterogeneous})}\label{subsubsec illustration hetergeneous}
We consider the problem formulation \eqref{eq: heterogeneous} where the action set $\A(x)$ is heterogeneous for different $x\in\X$. Note that $\A(x)$ is a compact set contained in a  $d_x-$dimensional subspace.
Given $x\in\X$, we choose $\{A_{x,i}\}_{i=1}^{d_x}$ as the barycentric spanner of $\A(x)$ and take $a_i=A_{x,i}$ for $i=1,\dots, d_x$. As stated in Statement \ref{statement counterfactual action divergence}, given $x\in\X$, the counterfactual action divergence between $a_t$ and $\{a_i\}_{i=1}^{t-1}$ evaluated at $x$ is
\begin{align*}
    V_x(a_t||\{a_i\}_{i=1}^{t-1}))=\kappa_x^2 b_{x,a_t}^{\top}(\sum_{i=1}^{t-1}b_{x,a_t}b_{x,a_t}^\top)^{-1}b_{x,a_t},
\end{align*} where $b_{x,a_t}$ is the coefficient vector of $a_t$ with respect to the basis $\{A_{x,i}\}_{i=1}^{d_x}$. From our previous result \eqref{eq: potential linear case} and the fact $\kappa_x\geq 1$, for all $T\geq 1$ and all sequences $\{a_i\}_{i=1}^{T}$ that satisfy $\{a_i\}_{i=1}^{d_x\land T}=\{A_{x,i}\}_{i=1}^{d_x\land T}$,
 \begin{align*}
       \sum_{t=1}^T 1\land \left[V_x(a_t||\{a_i\}_{i=1}^{t-1})\right]=\sum_{t=1}^T \left[1\land \kappa_x^2 b_{x,a_t}^\top(\sum_{i=1}^{t-1}[b_{x,a_i} b_{x,a_i}^\top])^{-1}b_{x,a_t}\right] \leq \kappa_x^2 3d_x\log T.
   \end{align*}
 By taking
   $\mathcal{E}_x=\kappa_x^2 d_x$, and  $\textup{poly}(\log T)=3\log T$, we verify Assumption \ref{asm per context} with $\mathcal{E}=\E[\kappa_x^2]d_x$. We note that under the heterogeneous formulation, Algorithm \ref{alg: uccb ia} needs to compute a different basis for each $\A(x)$, and the computation of counterfactual action divergence also requires a coefficient decomposition for each $x\in\X$.

 One significant advantage of Algorithm \ref{alg: uccb ia} is that the regret does not rely on the full dimension $d$---this means that we can increase feature context  as long as we can control the average decision entropy $\E[\kappa_x^2 d_x]$.

\section{Using ``optimistic subroutines" to generalize \texttt{FALCON}}\label{sec randomized}
What is the connection between our proposed optimistic algorithms and existing randomized algorithms? In this section, we show that by combining the idea of counterfactual confidence bounds and a non-trivial ``optimistic subroutine," we can also generalize \texttt{FALCON} to the infinite-action setting.   Note that the analysis of the resulting randomized algorithm is much more complex than the optimistic algorithm we introduced before. Through this extension, we see the simplicity and importance of the optimism principle for complex settings like infinite-action spaces.

The first offline-regression-oracle-efficient randomized algorithm in  general realizable contextual bandits, \texttt{FALCON} from  \cite{simchi2020bypassing},  is restricted to the finite-action setting. \texttt{FALCON} performs implicit optimization in policy space, but the allocation of policies reduces to a closed-form weighted allocation rule for actions. We find that it becomes more crucial to exploit the  counterfactual confidence bounds in the infinite-action setting: the optimization of weighted allocation rules  no longer has closed-form solutions, and we need to follow the foundational work \cite{agarwal2014taming} to design a coordinate-descent-based ``optimistic subroutine" to find feasible weighted allocations. Note that this subroutine is computationally much easier than the original one presented in \cite{agarwal2014taming} because it is decomposable across contexts, similar to the philosophy of the original \texttt{FALCON} algorithm.

As the required subroutine is a bit complex,  we focus on the linear action model \eqref{eq: linear semi} stated in Example \ref{example cb linear} for simplicity. Extensions to more complex models follow similar ideas, and the structure of the proposed algorithm remain mostly unchanged. We assume a  deterministic initialization oracle that outputs a barycentric spanner of the compact set $\A$ (e.g. the algorithm in \cite{awerbuch2004adaptive}), and an action maximization oracle that outputs the maximizer of the input function over $\A$. The following algorithm extends \texttt{FALCON} to the linear action model \eqref{eq: linear semi}, where the step 6 is a novel optimization problem to find the ``right" weighted allocation over actions. 
Here the ``$a^\top(\E_{\widetilde{a}\sim p_{t}}[\widetilde{a}\widetilde{a}^\top])^{-1}a$" term in \eqref{eq: optimization generalized falcon 2} is a continuous analogue to the counterfactual action divergence \eqref{eq: conterfactual action divergence linear coro}.

\begin{algorithm}[htbp]
\caption{a generalized version of \texttt{FALCON} for linear action model \eqref{eq: linear semi}}
\label{alg: generalized falcon}
{\textbf{Input}  epoch schedule $\{\tau_m\}_{m=1}^{\infty}$, $\tau_0=0$, tuning parameters $\{\beta_m\}_{m=1}^{\infty}$,  an arbitrary function  $\widehat{f}_1\in\F$.
}

\begin{algorithmic}[1]

\FOR{epoch $m=1,2,\dots$}
 \STATE{ Compute $\widehat{f}_m=\arg\min_{f\in\F}\sum_{t=1}^{\tau_m-1}(f(x_t,a_t)-r_t(x_t,a_t))^2$ via the least square oracle when $m\geq 1$.}
\FOR{round $t=\tau_{m-1}+1,\dots,\tau_m$}
    \STATE{Observe context $x_t$.}
    \STATE{Use the action maximization oracle to compute $\widehat{a}_t\in\max_{a\in\mathcal{A}}\widehat{f}_m(x_t,a)$.}
    \STATE{Run the algorithm   \texttt{OptimisticSubroutine}$(\A,\widehat{a}_t,\widehat{f}_m(x_t,\cdot),\beta_m)$ to find a distribution $p_t$ over $\A$  such that,}
        \begin{align}
       &\E_{a\sim p_t}[\widehat{f}_m(x_t,\widehat{a}_t)-\widehat{f}_m(x_t,a)]\leq {2\beta_m d}, \label{eq: optimization generalized falcon 1}\\
   \forall a\in\A, \quad & \widehat{f}_m(x_t, a)+\beta_m a^\top(\E_{\widetilde{a}\sim p_t}[\widetilde{a}\widetilde{a}^\top])^{-1}a\leq \widehat{f}_m(x_t,\widehat{a}_t) + {2\beta_m d}.\label{eq: optimization generalized falcon 2}
        \end{align}\vspace*{-0.2in}
        \STATE{Sample $a_t\sim p_t$ and observe reward $r_t(a_t)$.}
            \ENDFOR
\ENDFOR
\end{algorithmic}
\end{algorithm}

 Algorithm \ref{alg: generalized falcon} runs in an epoch schedule and only calls the least square oracle at
the pre-specified rounds $\tau_1, \tau_2,\dots$. We take $\tau_m=2^m$ for all $m\geq 1$ to simplify the statement of the theorem, though other choices of the epoch schedule are also possible \citep{simchi2020bypassing}.

\begin{theorem}[\bf{Regret of Algorithm \ref{alg: generalized falcon}}]\label{thm randomized}
Consider the problem formulation \eqref{eq: linear semi} stated in Example \ref{example cb linear}, under Assumption \ref{asm realizability}. Take the epoch schedule $\tau_m= 2^m$ for $m\geq 1$. Let $$\beta_m=30\sqrt{\log(|\F|\tau_{m-1}/\delta)/(2d\tau_{m-1})}$$ for $m=2,\dots$, and $\beta_1=1$. Then with probability at least $1-\delta$, for all $T\geq 1$, the regret of Algorithm \ref{alg: generalized falcon} after $T$ rounds is upper bounded by
$$
\textup{Regret}(T,\textup{Algorithm 3})\leq 608.5\sqrt{2dT\log(|\F|T/\delta)}+2\sqrt{2T\log(2/\delta)}+2.
$$
\end{theorem}

Theorem \ref{thm randomized} can be obtained by modifying the regret analysis of the original \texttt{FALCON} algorithm. (We refer the readers to \cite{simchi2020bypassing} for the background and intuition of the original \texttt{FALCON} algorithm, especially the ``Observation 2" in that paper.) However, the key challenge is to provide an efficient algorithm to find a weighted allocation rule that satisfy both \eqref{eq: optimization generalized falcon 1} and \eqref{eq: optimization generalized falcon 2} in Algorithm \ref{alg: generalized falcon}.

\begin{algorithm}[htbp]
\caption{\texttt{OptimisticSubroutine}$(\widehat{a}, \beta,\A, \widehat{h})$}
\label{alg: coordinate}
{\textbf{input} action set $\A$, greedy action $\widehat{a}\in\A$, function $\widehat{h}: \A\rightarrow [0,1]$, parameter $\beta>0$.}

\begin{algorithmic}[1]

\STATE{Obtain a barycentric spanner $\{A_i\}_{i=1}^{d}$ of $\A$ via the initialization oracle.}

\STATE{Set $q_0=\sum_{i=1}^d\frac{1}{d}\mathds{1}_{A_i}$.}
\FOR{$t=1,2,\dots$}
    \STATE{Set \begin{align}\label{eq: backtraking}
        q_{t-\frac{1}{2}}=\min\{\frac{2d}{2d+\E_{{a}\sim q}[(\widehat{h}(\widehat{a})-\widehat{h}({a}))/\beta]},1\}\cdot q_{t-1}.
    \end{align}}
    \STATE{Use the action maximization oracle to compute}
    \begin{align}\label{eq: optimistic subroutine}
        a_t=\argmax_{a\in\A} \left\{\widehat{h}(a)+\beta a^\top(\E_{\widetilde{a}\sim q_{t-\frac{1}{2}}}[\widetilde{a}\widetilde{a}^\top])^{-1}a\right\}.
    \end{align}
    \IF{ $\widehat{h}(a_t)+\beta a_t^\top(\E_{{a}\sim q_{t-\frac{1}{2}}}[{a}{a}^\top])^{-1}a_t>\widehat{h}({\widehat{a}})+2\beta d$,}
        \STATE{Run the coordinate descent step \begin{align}\label{eq: coordinate descent step}
        q_t=q_{t-\frac{1}{2}}+\frac{-2a_t^\top(\E_{{a}\sim q}[{a}{a}^\top])^{-1}a_t+2d+(\widehat{h}(\widehat{a})-\widehat{h}(a_t))/\beta}{ (a_t^\top(\E_{{a}\sim q}[{a}{a}^\top])^{-1}a_t)^2}\mathds{1}_{a_t}.
        \end{align}}
       \ELSE
       \STATE{Let $q_t=q_{t-\frac{1}{2}}$, halt and output \begin{align*}
           q_t+(1-\int_{\A}q_t(a)\textup{d}a)\mathds{1}_{\widehat{a}}.
       \end{align*}\vspace{-0.2in}}
\ENDIF
\ENDFOR
\end{algorithmic}
\end{algorithm}

We provide Algorithm \ref{alg: coordinate} as a subroutine to achieve this. The core idea of this algorithm is to use a coordinate descent procedure to compute a sparse distribution over actions, which is motivated by the optimization procedure used in   \cite{agarwal2014taming}---however, we extend their idea from the finite-action setting to the linear action model, which requires further matrix analysis and may be interesting in its own right. We call this algorithm \texttt{OptimisticSubroutine} as the algorithm is built upon the optimistic step \eqref{eq: optimistic subroutine}, where the ``$a^\top(\E_{\widetilde{a}\sim q_{t-\frac{1}{2}}}[\widetilde{a}\widetilde{a}^\top])^{-1}a$" term is a continuous analogue to the counterfactual action divergence \eqref{eq: conterfactual action divergence linear coro} in the linear action model.

\begin{proposition}[{\bf optimization through $\texttt{SubOpt}$}]\label{prop optimization}
At each round within epoch $m$, Algorithm \ref{alg: coordinate} outputs a probability distribution that satisfies \eqref{eq: optimization generalized falcon 1} and \eqref{eq: optimization generalized falcon 2} within at most $\lceil \frac{4}{\beta_m}+8d(\log d+1) \rceil$
iterations.
\end{proposition}

According to this proposition, the optimistic subroutine outputs an efficient solution that satisfies the requirements \eqref{eq: optimization generalized falcon 1} and \eqref{eq: optimization generalized falcon 2} within finite number of iterations at every rounds. As we can see, the 
design and analysis of the optimistic subroutine becomes challenging in the infinite-action setting, especially for complex problem formulations. In comparison, Algorithm \ref{alg: uccb ia} exhibits much cleaner structure and a principled analysis that covers many problem formulations of interest.

\section{Conclusion and future directions}\label{sec conclusion}
In this paper we propose $\texttt{UCCB}$, a simple generic principle to design optimistic algorithms in handling general function classes and large context spaces.  Key components of \texttt{UCCB} include: 1) building confidence bounds in policy space rather than in action space;  2) the potential function perspective that explains the power of optimism in the contextual setting; and 3) the natural extension to a unified framework for infinite-action general contextual bandits.  We present the first optimal and efficient optimistic algorithm for realizable contextual bandits with general function classes. Besides the traditional finite-action setting, we also discuss the infinite-action setting and provide the first solutions to many interesting models of practical interest.

Moving forward, there are many interesting future directions that may leverage the ideas presented in this work. 
 The principle of optimism in the face of uncertainty plays an essential role in  reinforcement learning. Currently the majority  of existing provably efficient algorithms are developed for the ``tabular" case, and their regret  scales with the cardinality of the state space. However, empirical reinforcement learning problems typically have a large state space and rely on function approximation \citep{jiang2017contextual}. Motivated by this challenge, a natural next step is to adapt the \texttt{UCCB} principle to
 reinforcement learning problems with large state space. This paper can be viewed as an initial step towards this goal, as the contextual MAB problem is a special case of episodic reinforcement learning where the episode length  is equal to one. 
 Within the scope of bandit problems, \texttt{UCB}-type algorithms are often the ``meta-algorithms" for many complex formulations when there is no contextual information. Since  \texttt{UCCB} improves over \texttt{UCB}-type algorithms in several fundamental contextual settings, this work may be a building block to combine contextual information and function approximation with more complex formulations  such as Gaussian process optimization \citep{srinivas2009gaussian}, bandits with long-term  constraints \citep{badanidiyuru2013bandits}, and bandits in non-stationary environments \citep{garivier2008upper}. We leave these directions to future work.

\bibliography{references}

\begin{thebibliography}{10}

\bibitem{abbasi2011improved}
Yasin Abbasi-Yadkori, D{\'a}vid P{\'a}l, and Csaba Szepesv{\'a}ri.
\newblock Improved algorithms for linear stochastic bandits.
\newblock In {\em Advances in Neural Information Processing Systems}, pages
  2312--2320, 2011.

\bibitem{abernethy2013large}
Jacob Abernethy, Kareem Amin, Michael Kearns, and Moez Draief.
\newblock Large-scale bandit problems and kwik learning.
\newblock In {\em International Conference on Machine Learning}, pages
  588--596, 2013.

\bibitem{agarwal2012contextual}
Alekh Agarwal, Miroslav Dud{\'\i}k, Satyen Kale, John Langford, and Robert
  Schapire.
\newblock Contextual bandit learning with predictable rewards.
\newblock In {\em Artificial Intelligence and Statistics}, pages 19--26, 2012.

\bibitem{agarwal2014taming}
Alekh Agarwal, Daniel Hsu, Satyen Kale, John Langford, Lihong Li, and Robert
  Schapire.
\newblock Taming the monster: A fast and simple algorithm for contextual
  bandits.
\newblock In {\em International Conference on Machine Learning}, pages
  1638--1646, 2014.

\bibitem{auer2002nonstochastic}
Peter Auer, Nicolo Cesa-Bianchi, Yoav Freund, and Robert~E Schapire.
\newblock The nonstochastic multiarmed bandit problem.
\newblock {\em SIAM journal on computing}, 32(1):48--77, 2002.

\bibitem{awerbuch2004adaptive}
Baruch Awerbuch and Robert~D Kleinberg.
\newblock Adaptive routing with end-to-end feedback: Distributed learning and
  geometric approaches.
\newblock In {\em Proceedings of the thirty-sixth annual ACM symposium on
  Theory of computing}, pages 45--53, 2004.

\bibitem{badanidiyuru2013bandits}
Ashwinkumar Badanidiyuru, Robert Kleinberg, and Aleksandrs Slivkins.
\newblock Bandits with knapsacks.
\newblock In {\em 2013 IEEE 54th Annual Symposium on Foundations of Computer
  Science}, pages 207--216. IEEE, 2013.

\bibitem{bartlett2008high}
Peter~L Bartlett, Varsha Dani, Thomas Hayes, Sham Kakade, Alexander Rakhlin,
  and Ambuj Tewari.
\newblock High-probability regret bounds for bandit online linear optimization.
\newblock 2008.

\bibitem{bietti2018contextual}
Alberto Bietti, Alekh Agarwal, and John Langford.
\newblock A contextual bandit bake-off.
\newblock {\em arXiv preprint arXiv:1802.04064}, 2018.

\bibitem{chernozhukov2019semi}
Victor Chernozhukov, Mert Demirer, Greg Lewis, and Vasilis Syrgkanis.
\newblock Semi-parametric efficient policy learning with continuous actions.
\newblock In {\em Advances in Neural Information Processing Systems}, pages
  15039--15049, 2019.

\bibitem{chu2011contextual}
Wei Chu, Lihong Li, Lev Reyzin, and Robert Schapire.
\newblock Contextual bandits with linear payoff functions.
\newblock In {\em Proceedings of the Fourteenth International Conference on
  Artificial Intelligence and Statistics}, pages 208--214, 2011.

\bibitem{dani2008stochastic}
Varsha Dani, Thomas~P Hayes, and Sham~M Kakade.
\newblock Stochastic linear optimization under bandit feedback.
\newblock 2008.

\bibitem{dudik2011efficient}
Miroslav Dudik, Daniel Hsu, Satyen Kale, Nikos Karampatziakis, John Langford,
  Lev Reyzin, and Tong Zhang.
\newblock Efficient optimal learning for contextual bandits.
\newblock {\em arXiv preprint arXiv:1106.2369}, 2011.

\bibitem{filippi2010parametric}
Sarah Filippi, Olivier Cappe, Aur{\'e}lien Garivier, and Csaba Szepesv{\'a}ri.
\newblock Parametric bandits: The generalized linear case.
\newblock In {\em Advances in Neural Information Processing Systems}, pages
  586--594, 2010.

\bibitem{foster2018practical}
Dylan~J Foster, Alekh Agarwal, Miroslav Dud{\'\i}k, Haipeng Luo, and Robert~E
  Schapire.
\newblock Practical contextual bandits with regression oracles.
\newblock {\em arXiv preprint arXiv:1803.01088}, 2018.

\bibitem{foster2020adapting}
Dylan~J Foster, Claudio Gentile, Mehryar Mohri, and Julian Zimmert.
\newblock Adapting to misspecification in contextual bandits.
\newblock {\em Advances in Neural Information Processing Systems},
  33:11478--11489, 2020.

\bibitem{foster2020beyond}
Dylan~J Foster and Alexander Rakhlin.
\newblock Beyond ucb: Optimal and efficient contextual bandits with regression
  oracles.
\newblock {\em arXiv preprint arXiv:2002.04926}, 2020.

\bibitem{garivier2008upper}
Aur{\'e}lien Garivier and Eric Moulines.
\newblock On upper-confidence bound policies for non-stationary bandit
  problems.
\newblock {\em arXiv preprint arXiv:0805.3415}, 2008.

\bibitem{jiang2017contextual}
Nan Jiang, Akshay Krishnamurthy, Alekh Agarwal, John Langford, and Robert~E
  Schapire.
\newblock Contextual decision processes with low bellman rank are
  pac-learnable.
\newblock In {\em International Conference on Machine Learning}, pages
  1704--1713, 2017.

\bibitem{kleinberg2005nearly}
Robert~D Kleinberg.
\newblock Nearly tight bounds for the continuum-armed bandit problem.
\newblock In {\em Advances in Neural Information Processing Systems}, pages
  697--704, 2005.

\bibitem{krishnamurthy2019active}
Akshay Krishnamurthy, Alekh Agarwal, Tzu-Kuo Huang, Hal Daum{\'e}~III, and John
  Langford.
\newblock Active learning for cost-sensitive classification.
\newblock {\em Journal of Machine Learning Research}, 20(65):1--50, 2019.

\bibitem{krishnamurthy2019contextual}
Akshay Krishnamurthy, John Langford, Aleksandrs Slivkins, and Chicheng Zhang.
\newblock Contextual bandits with continuous actions: Smoothing, zooming, and
  adapting.
\newblock {\em arXiv preprint arXiv:1902.01520}, 2019.

\bibitem{langford2007epoch}
John Langford and Tong Zhang.
\newblock The epoch-greedy algorithm for multi-armed bandits with side
  information.
\newblock {\em Advances in neural information processing systems}, 20:817--824,
  2007.

\bibitem{levy2022counterfactual}
Orin Levy, Asaf Cassel, Alon Cohen, and Yishay Mansour.
\newblock Counterfactual optimism: Rate optimal regret for stochastic
  contextual mdps.
\newblock {\em arXiv preprint arXiv:2211.14932}, 2022.

\bibitem{levy2023optimism}
Orin Levy and Yishay Mansour.
\newblock Optimism in face of a context: Regret guarantees for stochastic
  contextual mdp.
\newblock In {\em Proceedings of the AAAI Conference on Artificial
  Intelligence}, volume~37, pages 8510--8517, 2023.

\bibitem{li2017provably}
Lihong Li, Yu~Lu, and Dengyong Zhou.
\newblock Provably optimal algorithms for generalized linear contextual
  bandits.
\newblock In {\em Proceedings of the 34th International Conference on Machine
  Learning-Volume 70}, pages 2071--2080. JMLR. org, 2017.

\bibitem{rigollet2010nonparametric}
Philippe Rigollet and Assaf Zeevi.
\newblock Nonparametric bandits with covariates.
\newblock {\em arXiv preprint arXiv:1003.1630}, 2010.

\bibitem{russo2013eluder}
Daniel Russo and Benjamin Van~Roy.
\newblock Eluder dimension and the sample complexity of optimistic exploration.
\newblock In {\em Advances in Neural Information Processing Systems}, pages
  2256--2264, 2013.

\bibitem{simchi2020bypassing}
David Simchi-Levi and Yunzong Xu.
\newblock Bypassing the monster: A faster and simpler optimal algorithm for
  contextual bandits under realizability.
\newblock {\em arXiv preprint arXiv:2003.12699}, 2020.

\bibitem{slivkins2014contextual}
Aleksandrs Slivkins.
\newblock Contextual bandits with similarity information.
\newblock {\em The Journal of Machine Learning Research}, 15(1):2533--2568,
  2014.

\bibitem{slivkins2019introduction}
Aleksandrs Slivkins et~al.
\newblock Introduction to multi-armed bandits.
\newblock {\em Foundations and Trends{\textregistered} in Machine Learning},
  12(1-2):1--286, 2019.

\bibitem{srinivas2009gaussian}
Niranjan Srinivas, Andreas Krause, Sham~M Kakade, and Matthias Seeger.
\newblock Gaussian process optimization in the bandit setting: No regret and
  experimental design.
\newblock {\em arXiv preprint arXiv:0912.3995}, 2009.

\bibitem{vershynin2010introduction}
Roman Vershynin.
\newblock Introduction to the non-asymptotic analysis of random matrices.
\newblock {\em arXiv preprint arXiv:1011.3027}, 2010.

\end{thebibliography}

\appendix
\section{Proofs for the finite-action setting}\label{sec proof finite action}

\subsection{Proof of Theorem \ref{thm uccb}. }
We prove the theorem on the clean event stated in Lemma \ref{lemma square trick}, whose measure is at least $1-\delta/2$. For all $t>K$,
\begin{align}\label{eq: thm1 1}
    \E_x[f^*(x,\pi_{f^*}(x)]\leq \E_x[\widehat{f}_t(x, \pi_{f^*}(x))]+{\E_x\big[\frac{\beta_t}{\sum_{i=1}^t{\mathds{1}\{\pi_{f^*}(x)=\pi_i(x)\}}}\big]}+\frac{K\beta_t}{t}\nonumber\\
    \leq \argmax_{\pi\in\Pi}\bigg\{\E_x[\widehat{f}_t(x, \pi(x))]+{\E_x\big[\frac{\beta_t}{\sum_{i=1}^t{\mathds{1}\{\pi(x)=\pi_i(x)\}}}\big]}\bigg\}+\frac{K\beta_t}{t}\nonumber\\
    =\E_x[\widehat{f}_t(x, \pi_{t}(x))]+{\E_x\big[\frac{\beta_t}{\sum_{i=1}^t{\mathds{1}\{\pi_{t}(x)=\pi_i(x)\}}}\big]}+\frac{K\beta_t}{t}\nonumber\\
    \leq \E_x[f^*(x,\pi_t(x))]+{\E_x\big[\frac{2\beta_t}{\sum_{i=1}^t{\mathds{1}\{\pi_{t}(x)=\pi_i(x)\}}}\big]}+\frac{2K\beta_t}{t},
\end{align}
where the first and the last inequality are due to Lemma \ref{lemma square trick}; the second inequality due to maximization over policies.

Therefore, we have the following:
\begin{align}\label{eq: used in proof of thm}
   \sum_{t=1}^T \E[f^*(x_t,\pi_{f^*}(x_t))-f^*(x_t,a_t)|H_{t-1}]= \sum_{t=1}^T \big(\E_x[f^*(x,\pi_{f^*}(x))]-\E_{x}[f^*(x,\pi_t(x))]\big)\nonumber\\
   \leq \sum_{t=K+1}^T{\E_x\big[\frac{2\beta_t}{\sum_{i=1}^t{\mathds{1}\{\pi_{t}(x)=\pi_i(x)\}}}\big]}+\sum_{t=K+1}^T\frac{2K\beta_t}{t}+K
   \nonumber \\ \leq 2\beta_T \sum_{t=K+1}^T\E_x\big[\frac{1}{\sum_{i=1}^{t-1}\mathds{1}\{\pi_t(x)=\pi_i(x)\}}\big]+2\sqrt{17KT\log(|\F|T^3/\delta)}+K\nonumber\\
    \leq 2\sqrt{17KT\log(2|\F|T^3/\delta)}(\log(T/K)+1)+K,
\end{align}
where the first line uses the equivalence proved in Lemma \ref{lemma equivalence}; the second line is due to \eqref{eq: thm1 1}; the third line is due to $\beta_t\leq\beta_T$ and $\sum_{K+1}^T 1/\sqrt{t}\leq \sqrt{T}$; and the last line is due to the contextual potential lemma (Lemma \ref{lemma contextual potential lemma}).

By Azuma's inequality, with probability at least $1-\delta/2$, we can bound the regret by 
\begin{align}\label{eq: azuma}
    \text{Regret}(T, \text{Algorithm \ref{alg: uccb}})\leq  \sum_{t=1}^T \E[f^*(x_t,\pi_{f^*}(x_t))-f^*(x_t,a_t)|H_{t-1}]+\sqrt{2T\log(2/\delta)}.
\end{align} Therefore, by a union bound and inequalities \eqref{eq: used in proof of thm} \eqref{eq: azuma}, with probability at least $1-\delta$, the regret of Algorithm $\ref{alg: uccb}$  after $T$ rounds is upper bounded by
\begin{align*}
   \textup{Regret}(T, \text{Algorithm \ref{alg: uccb}})\leq  2\sqrt{17KT\log(2|\F|T^3/\delta)}(\log(T/K)+1)+\sqrt{2T\log(2/\delta)}+K.
\end{align*}

\hfill$\square$

\subsection{Analysis on the confidence}\label{subsec analysis confidence}

The main goal of this subsection is to prove Lemma \ref{lemma estimation error regressor}.  For a fixed $f$, we denote $Y_{f,i}=(f(x_i,a_i)-r_i(x_i,a_i))^2-(f^*(x_i,a_i)-r_i(x_i,a_i))^2$, $i=1,2,\dots$.

\subsubsection{Proof of Lemma \ref{lemma estimation error regressor}.}
 For a fixed $f\in\F$, when conditioned on $\Upsilon_{i-1}$,  we have
\begin{align*} 
\E_{x_i,a_i}\left[(f(x_i,a_i)-f^*(x_i,a_i))^2|H_{i-1}\right]=\E_{x_i}\left[(f(x_i,\pi_i(x_i))-f^*(x_i,\pi_i(x_i)))^2|H_{i-1}\right]\\
=\E_{x}\left[(f(x,\pi_i(x))-f^*(x,\pi_i(x)))^2\right|H_{i-1}]\\
=\E_{x}\left[(f(x,\pi_i(x))-f^*(x,\pi_i(x)))^2\right],
\end{align*}
where the first equation is because $a_i=\pi_i(x_i)$ and the fact that $\pi_i$ is completely determined by $H_{t-1}$; the second equation is because the independence between $x_i$ and $H_{i-1}$; and the third inequality is because $(f(x_i,\pi_i(x))-f^*(x_i,\pi_i(x_i)))^2$ depends on $H_{i-1}$ only through $\pi_i$.

Therefore, 
 \begin{equation*} 
\sum_{i=1}^{{t-1}}\E_{x_i,a_i}\left[(f(x_i,a_i)-f^*(x_i,a_i))^2|H_{i-1}\right]=\sum_{i=1}^{{t-1}}\E_{x}\left[(f(x,\pi_i(x))-f^*(x_i,\pi_i(x)))^2\right].
\end{equation*}
Applying Lemma \ref{lemma uniform convergence time}, we know that $\forall \delta\in (0,1)$, with probability at least $1-{\delta}/2$, 
\begin{equation}\label{eq: clean-event 2}
\sum_{i=1}^{{t-1}}\E_{x}\left[(f_t(x,\pi_i(x))-f^*(x_i,\pi_i(x)))^2\right]\le  68\log({2|\F|t^3}/{\delta}) +2\sum_{i=1}^{t-1}Y_{f_t,i},
\end{equation}
uniformly over all $t\geq K$ and  all fixed sequence $ f_{K}, f_{K+1}, \dots\in\F$.

Therefore, $\forall \pi\in\Pi$,
\begin{align}\label{eq: before cauchy}
   \E_x \Big[\sum_{i=1}^{t-1} \mathds{1}\{\pi(x)=\pi_i(x)\}(f_t(x,\pi(x))-f^*(x,\pi(x)))^2\Big]\nonumber\\
   =\E_x \Big[\sum_{i=1}^{t-1} \mathds{1}\{\pi(x)=\pi_i(x)\}(f_t(x,\pi_i(x))-f^*(x,\pi_i(x)))^2\Big]\nonumber\\
   =\sum_{i=1}^{t-1}\E_x[\mathds{1}\{\pi(x)=\pi_i(x)\}(f_t(x,\pi_i(x))-f^*(x,\pi_i(x)))^2]\nonumber\\
   \leq \sum_{i=1}^{t-1}\E_x[(f_t(x,\pi_i(x))-f^*(x,\pi_i(x)))^2]\nonumber\\
   \leq 68\log({2|\F|t^3}/{\delta}) +2\sum_{i=1}^{t-1}Y_{f_t,i},
\end{align}
where the first inequalities are due to $\mathds{1}\{\pi(x)=\pi_i(x)\}\leq 1$ and the second inequality is \eqref{eq: clean-event 2}.

Since Algorithm \ref{alg: uccb} pick all actions exactly once during the first $K$ rounds,   $t>K$  will ensure $\sum_{i=1}^{t-1}\mathds{1}\{\pi(x)=\pi(x)\}\geq 1, \forall x\in \X$. 

From Cauchy-Schwarz's inequality, $\forall t>K$,  $\forall \pi\in\Pi$, 
\begin{align*}
    |\E_x[f_t(x,\pi(x))-f^*(x,\pi(x))]|\\\leq \sqrt{\E_x\big[\frac{1}{\sum_{i=1}^{t-1}{\mathds{1}\{\pi(x)=\pi_i(x)\}}}\big]}\sqrt{\E_x \Big[\sum_{i=1}^{t-1} \mathds{1}\{\pi(x)=\pi_i(x)\}(f_t(x,\pi(x))-f^*(x_i,\pi(x)))^2\Big]}.
\end{align*}
Combine the above inequality with \eqref{eq: before cauchy}, we prove 
\begin{align*}
    |\E_x[f_t(x,\pi(x))-f^*(x,\pi(x))]|\\\leq \sqrt{\E_x\big[\frac{1}{\sum_{i=1}^{t-1}{\mathds{1}\{\pi(x)=\pi_i(x)\}}}\big]}\sqrt{68\log({2|\F|t^3}/{\delta}) +2\sum_{i=1}^{t-1}Y_{f_t,i}}.
\end{align*}
Taking $f_t=\widehat{f}_t$ in the above inequality, and use the fact $\sum_{i=1}^{t-1}Y_{\widehat{f}_t,i}\leq 0$ (as the least square solution $\widehat{f}_t$ minimizes $\sum_{i=1}^{t-1}(f(x_i,a_i)-r_i(x_i,a_i))^2$), we obtain: with probability at least $1-\delta/2$, $\forall t>K$, $\forall\pi\in\Pi$,
\begin{align*}
    \big|\E_x[\widehat{f}_t(x,\pi(x))]-\E_x [f^*(x,\pi(x))]\big|\\\leq \sqrt{\E_x\big[\frac{1}{\sum_{i=1}^{t-1}{\mathds{1}\{\pi(x)=\pi_i(x)\}}}\big]}\sqrt{68\log({2|\F|t^3}/{\delta})}
    \\\leq \sqrt{\E_x\big[\frac{1}{\sum_{i=1}^{t-1}{\mathds{1}\{\pi(x)=\pi_i(x)\}}}\big]}\sqrt{68\log({2|\F|t^3}/{\delta})}
\end{align*}

\hfill$\square$

\paragraph{}\ 

\subsubsection{Proof of Lemma \ref{lemma uniform convergence time}}
 We now prove Lemma \ref{lemma uniform convergence time} and the supporting lemmas required to prove Lemma \ref{lemma uniform convergence time}.

\paragraph{Proof of Lemma \ref{lemma uniform convergence time}.}
Fix a $\delta\in (0,1)$. Take $\delta_t=\delta/2t^3$, and apply a union bound to Lemma \ref{lemma uniform convergence} with all $t\geq 2$. From 
 \begin{align*}
    \sum_{t=1}^{\infty} \delta_t \log_2 (t-1)\leq   \sum_{t=2}^{\infty} \delta/{2t^2}\leq \delta/2,
 \end{align*} we know that with probability at least $1-\delta/2$,
\begin{align*}
\sum_{i=1}^{t-1}\E_{x_i,a_i}\left[(f_{t}(x_i,a_i)-f^*(x_i,a_i))^2|H_{i-1}\right]\le 68\log({2|\F|t^3}/{\delta})+2\sum_{i=1}^{t-1}Y_{f_{t},i},
\end{align*}
uniformly over all $t\geq 2$ and all fixed sequence $f_2, f_3, \dots \in\F$.

\hfill$\square$

\paragraph{}
\ 

\begin{lemma}[{\bf uniform convergence over $\F$}]\label{lemma uniform convergence} For a fixed $t\geq 2$ and a fixed $\delta_t\in (0,1/e^2)$, with probability at least $1-\log_2(t-1){\delta_t}$, 
we have
\begin{align}\label{eq: lemma result uniform convergence}
\sum_{i=1}^{t-1}\E_{x_i,a_i}\left[(f(x_i,a_i)-f^*(x_i,a_i))^2|H_{i-1}\right]\le  68\log({|\F|}/{\delta_t}) +2\sum_{i=1}^{t-1}Y_{f,i},
\end{align}
uniformly over all $f\in\F$.
\end{lemma}

\paragraph{Proof of Lemma \ref{lemma uniform convergence}.}
We have $|Y_{f,i}|\leq 1, \forall i$. From Lemma \ref{lemma freeman}, for $\delta_t/|\F|\leq \delta_t<1/e^2 $, with probability at least $1-\log_2 (t-1)\delta_t/|\F|$,
\begin{align*}
    \sum_{i=1}^{t-1}\E[Y_{f,i}|H_{i-1}]-\sum_{i=1}^{t-1}Y_{f,i}\leq 4\sqrt{\sum_{i=1}^{t-1}\Var[Y_{f,i}|H_{i-1}]\log (|\F|/\delta_t)}+2\log (|\F|/\delta_t).
\end{align*}
Applying union bound to all $f\in\F$, we obtain that with probability at least $1-\log_2 (t-1)\delta_t\geq 1-\log_2 t\delta_t$,
    \begin{align*}
    \sum_{i=1}^{t-1}\E[Y_{f,i}|H_{i-1}]-\sum_{i=1}^{t-1}Y_{f,i}\leq 4\sqrt{\sum_{i=1}^{t-1}\Var[Y_{f,i}|H_{i-1}]\log (|\F|/\delta_t)}+2\log (|\F|/\delta_t), \quad \forall f\in\F.
\end{align*}
From Lemma \ref{lemma ag12-1} we have $\Var[Y_{f,i}|H_i]\leq 4\E[Y_{f,i}|H_i]$. Therefore
   \begin{align*}
    \sum_{i=1}^{t-1}\E[Y_{f,i}|H_{i-1}]\leq 4\sqrt{\sum_{i=1}^{t-1}\Var[Y_{f,i}|H_{i-1}]\log (|\F|/\delta_t)}+2\log (|\F|/\delta_t)+\sum_{i=1}^{t-1}Y_{f,i}\\
    \leq 8\sqrt{\sum_{i=1}^{t-1}\E[Y_{f,i}|H_{i-1}]\log (|\F|/\delta_t)}+2\log (|\F|/\delta_t)+\sum_{i=1}^{t-1}Y_{f,i}, \quad \forall f\in\F.
\end{align*}
This implies $\forall f\in\F$,
\begin{align*}
     \left(\sqrt{\sum_{i=1}^{t-1}\E[Y_{f,i}|H_{i-1}]}- 4 \sqrt{\log(|\F|/\delta_t)}\right)^2\leq {18\log(|\F|/\delta_t)+\sum_{i=1}^{t-1}Y_{f,i}},
\end{align*}
which further implies $\forall f\in\F$,
\begin{align*}
    \sum_{i=1}^{t-1}\E[Y_{f,i}|H_{i-1}]\leq 68{\log(|\F|/\delta_t)}+2\sum_{i=1}^{t-1}Y_{f,i}.
\end{align*}

From Lemma \ref{lemma ag12-1}, we have
\begin{align*}
    \sum_{i=1}^{t-1}\E_{x_i,a_i}\left[f(x_i,a_i)-f^*(x_i,a_i))^2|H_{i-1}\right]=\sum_{i=1}^{t-1}\E[Y_{f,i}|H_{i-1}]\leq 68{\log(|\F|/\delta_t)}+2\sum_{i=1}^{t-1}Y_{f,i}.
\end{align*}
This finish the proof to Lemma \ref{lemma uniform convergence}.

\hfill$\square$

\paragraph{}

The following two lemmas are used in the proof of Lemma \ref{lemma uniform convergence}.

\begin{lemma}[\bf{Freeman's inequality, \cite{bartlett2008high}}]\label{lemma freeman}
Suppose $Z_1, Z_2,\dots, Z_t$ is a martingale difference sequence with $|Z_i|\leq b$ for all $i=1,\dots, t$.  Then for any $\delta<1/e^2$,  with probability at least $1-(\log_2t)\delta$, 
  \begin{align*}
     \sum_{i=1}^t Z_i\leq 4\sqrt{\sum_{i=1}^t \Var[Z_i|Z_1,\dots,Z_{i-1}]\log(1/\delta)}+2b\log(1/\delta).
 \end{align*}
 \end{lemma}
 
 \paragraph{}
 
 \begin{lemma}[\bf{Lemma 4.2 in \cite{agarwal2012contextual}}]\label{lemma ag12-1}
Fix a function $f\in\F$. Suppose we sample $x$ from the data distribution $\mathcal{D}_{\X}$, and $r(x,a)$ from $\D_{x,a}$. Define the random variable
$$
Y=(f(x,a)-r(x,a))^2-(f^*(x,a)-r(x,a))^2.
$$
Then we have
$$
\E_{x,r,a}[Y]=\E_{x,a}[(f(x,a)-f^*(x,a))^2],
$$
$$
\Var_{x,r,a}[Y]\le 4\E_{x,r,a}[Y].
$$
\end{lemma}

\subsection{Proof of Lemma \ref{lemma contextual potential lemma}}

\paragraph{\bf Proof of Lemma \ref{lemma contextual potential lemma}.} For any fixed $x\in\X$, we have
\begin{align*}
    \sum_{t=K+1}^T\frac{1}{\sum_{i=1}^{t-1}\mathds{1}\{\pi_t(x)=\pi_j(x)\}}\leq \sum_{a\in\A}\sum_{i=1}^{\sum_{t=1}^T\mathds{1}\{\pi_t(x)=a\}}\frac{1}{i}
    \\\leq \sum_{a\in\A}(1+\log(\sum_{t=1}^T\mathds{1}\{\pi_t(x)=a\}))
    \leq K+K\log(T/K),
\end{align*}
where the last inequality is due to Jensen's inequality. By taking expectation on both sides of the above inequality, we prove the lemma.

\hfill$\square$

\section{Proofs for the extensions to infinite function classes}
\subsection{Proof of Corollary \ref{coro parametric}}\label{subsec infinite class}
From the well-known result on the covering of $d-$dimensional balls \cite{vershynin2010introduction}, the covering number of a $d-$dimensional ball with radius $\frac{\Delta}{2}$ and discretization error $\frac{1}{Lt}$ is bounded by $(1+{\Delta}{Lt})^d$, so there exists a set $V_t$ of size no more than $(1+{\Delta}{Lt})^d+1\leq (2+{\Delta}{Lt})^d$ that contains $\theta^*$ and satisfies 
\begin{align*}
  \forall \theta\in\Theta\ \exists v\in V_t\text{ s.t. }  \|\theta-v\|\leq \frac{1}{Lt}.
\end{align*}
We see $\log |V_t|\leq d\log(2+{\Delta}{Lt})$. 
$\forall f_\theta\in\F, x\in\X, a\in\A$, take $v$ to be the closest point to $\theta$ in $V_t$, we have
\begin{align*}
    (f_{\theta}(x,a)-f^*(x,a))^2=(f_{\theta}(x,a)-f_{v}(x,a)+f_v(x,a)-f_{\theta^*}(x,a))^2\\\leq 2(f_{\theta}(x,a)-f_{v}(x,a))^2+2(f_v(x,a)-f_{\theta^*}(x,a))^2\\
    \leq 2L^2\|\theta-v\|^2+2(f_v(x,a)-f_{\theta^*}(x,a))^2\\\leq \frac{2}{t^2}+2(f_v(x,a)-f_{\theta^*}(x,a))^2.
\end{align*}

Sine $V_t$ is a finite function class, we can prove a slight modification of Lemma \ref{lemma uniform convergence}, with the result \eqref{eq: lemma result uniform convergence} becomes
\begin{align*}
    \sum_{i=1}^{t-1}\E_{x_i,a_i}[(f(x_i,a_i)-f^*(x_i,a_i))^2|H_{i-1}]\leq 136\log (|V_t|/\delta_t)+2+4\sum_{i=1}^{t-1}Y_{f,i}.
\end{align*}
Following the same path in the proof of Lemma \ref{lemma uniform convergence time}, we can prove a slight modification of Lemma \ref{lemma uniform convergence time}: with probability at least $1-\frac{\delta}{2}$,
\begin{align*}
\sum_{i=1}^{t-1}\E_{x_i,a_i}\left[(f_{t}(x_i,a_i)-f^*(x_i,a_i))^2|H_{i-1}\right]\leq  136\Big(d\log(2+{\Delta}{Lt})+\log\frac{2t^3}{\delta}\Big)+2+4\sum_{i=1}^{t-1}Y_{f_t,i}\\
\leq 136\Big(d\log(2+{\Delta}{Lt})+\log\frac{2t^3}{\delta}+1\Big)+4\sum_{i=1}^{t-1}Y_{f_t,i},
\end{align*}
uniformly over all $t\geq 2$ and all fixed sequence $f_2, f_3, \dots \in\F$.

 By setting the parameter $\beta_t$ to be \begin{align*}
     \beta_t=\sqrt{34t/K}\sqrt{d\log(2+{\Delta Lt})+\log({2t^3}/{\delta})+1}
 \end{align*} in Algorithm \ref{alg: uccb}, we can prove a slight modification of Theorem \ref{thm uccb}, with the result being
 \begin{align*}
     \text{Regret}(T, \text{Algorithm \ref{alg: uccb}})\leq
     2K\beta_T (\log(T/K)+1)+\sqrt{2T\log(2/\delta)}+K.
 \end{align*}

\subsection{Proof of Corollary \ref{coro general class}}
We introduce the following Lemma adopt from \cite{foster2018practical}:
\begin{lemma}[\bf{a consequence of Lemma 4 in \cite{foster2018practical}}]
$\forall \delta\in(0,1)$, with probability at least $1-\delta$, 
\begin{align*}
    \sum_{i=\tau_1}^{\tau_2}\E_{x_i,a_i}[(f(x_i,a_i)-f^*(x_i,a_i))^2|H_{i-1}]\leq 2 \sum_{i=\tau_1}^{\tau_2}Y_{f,i}+\\
    K\cdot\inf_{\eps>0}\left\{100\eps T+320\log\left(\frac{4  K\E_{\{x_i\}_{i=1}^{T}}\mathcal{N}_1(\G, \eps, \{x_i\}_{i=1}^{T})}{\delta}\right)\right\}.
\end{align*}
for all $1\leq \tau_1\leq\tau_2\leq T$ and $g\in\G$.
\end{lemma}
 
We then prove a slight modification of Lemma \ref{lemma uniform convergence time}, with the result \eqref{eq: uniform over sequence} becomes
\begin{align*}
\sum_{i=1}^{t-1}\E_{x_i,a_i}\left[(f_{t}(x_i,a_i)-f^*(x_i,a_i))^2|H_{i-1}\right]\le  2\sum_{i=1}^{t-1}Y_{f_t,i}+\\
K\cdot\inf_{\eps>0}\left\{100\eps T+320\log\left(\frac{8K T^3 \E_{\{x_i\}_{i=1}^{T}}\mathcal{N}_1(\G, \eps, \{x_i\}_{i=1}^{T})}{\delta}\right)\right\},
\end{align*}
uniformly over all $t\geq 2$ and all fixed sequence $f_2, f_3, \dots \in\F$.

 By setting the parameter $\beta_t$ in Algorithm \ref{alg: uccb} to be the fixed value $$\beta=\sqrt{TK}\cdot\inf_{\eps>0}\left\{25\eps T+80\log\left(\frac{8K T^3 \E_{\{x_i\}_{i=1}^{T}}\mathcal{N}_1(\G, \eps, \{x_i\}_{i=1}^{T})}{\delta}\right)\right\},$$ we can prove a slight modification of Theorem \ref{thm uccb}, with the result being 
 \begin{align*}
     \text{Regret}(T, \text{Algorithm \ref{alg: uccb}})\leq 
     2K\beta(\log(T/K)+1)+\sqrt{2T\log(2/\delta)}+K.
 \end{align*}

\section{Proofs for the infinite-action setting}\label{sec proof infinite action}

\paragraph{Proof of Theorem \ref{thm uccb infinite}.}
We prove the theorem on the clean event stated in Lemma \ref{lemma uniform time infinite}, whose measure is at least $1-\delta/2$. 
For all $t\geq 2$,
\begin{align*}
    \E_x[f^*(x,\pi_{f^*}(x)]\leq  \E_x[1\land\beta_t{V_x(\pi_{f^*}(x)||\{\pi_i(x)\}_{i=1}^{t-1}) }] +\mathcal{E}\beta_t/t\nonumber\\
  \leq \E_x[\widehat{f}_t(x, \pi_{t}(x))]+ \E_x[1\land\beta_t{V_x(\pi_t(x)||\{\pi_i(x)\}_{i=1}^{t-1}) }] +\mathcal{E}\beta_t/t\nonumber\\
    \leq \E_x[f^*(x,\pi_t(x))]+2\E_x[1\land\beta_t{V_x(\pi_t(x)||\{\pi_i(x)\}_{i=1}^{t-1}) }] +2\mathcal{E}\beta_t/t.
\end{align*}
where the first and the last inequalities are due to Lemma \ref{lemma uniform time infinite}; the second inequality is due to the definition of $\pi_t$ in Lemma \ref{lemma equivalence infinite action}.
The above argument implies that for all $t\geq 2$
\begin{align}\label{eq: thm2 1}
\E_x[f^*(x,\pi_{f^*}(x)-f^*(x,\pi_t(x)]\leq  2\E_x[1\land\beta_t{V_x(\pi_t(x)||\{\pi_i(x)\}_{i=1}^{t-1}) }] +2\mathcal{E}\beta_t/t.
\end{align}
When $t=1$, inequality \eqref{eq: thm2 1} trivially holds true, because $ V_x(\pi_1(x)||\emptyset)=\infty$ by definition. So inequality \eqref{eq: thm2 1} holds true for all $t\geq 1$.

When $T\leq\mathcal{E}$, we can bound the regret by $\mathcal{E}$. We now give the regret bound for the case $T> \mathcal{E}$. We have the following:
\begin{align}\label{eq: used in proof of thm infinite}
   \sum_{t=1}^T \E[f^*(x_t,\pi_{f^*}(x_t))-f^*(x_t,a_t)|H_{t-1}]= \sum_{t=1}^T \big(\E_x[f^*(x,\pi_{f^*}(x))]-\E_{x}[f^*(x,\pi_t(x))]\big)\nonumber\\
   \leq \sum_{t=1}^T2\E_x[1\land\beta_t{V_x(\pi_t(x)||\{\pi_i(x)\}_{i=1}^{t-1}) }] +\sum_{t=1}^T2\mathcal{E}\beta_t/t\nonumber\\
   \leq 2\sum_{t=1}^T\E_x[1\land\beta_t{V_x(\pi_t(x)||\{\pi_i(x)\}_{i=1}^{t-1}) }]+2\sqrt{17\mathcal{E}T\log(2|\F|T^3/\delta)}
   \nonumber\\\leq 2\sum_{t=1}^T\E_x[\beta_T\land\beta_T{V_x(\pi_t(x)||\{\pi_i(x)\}_{i=1}^{t-1}) }] +2\sqrt{17\mathcal{E}T\log(2|\F|T^3/\delta)}\nonumber\\
   =  {2\beta_T\E_x[\sum_{t=1}^T1\land V_x(\pi_t(x)||\{\pi_i(x)\}_{i=1}^{t-1})]}+2\sqrt{17\mathcal{E}T\log(2|\F|T^3/\delta)}
   \nonumber \\ \leq 2\beta_T\mathcal{E}\textup{poly}(\log T) +2\sqrt{17\mathcal{E}T\log(2|\F|T^3/\delta)}\nonumber\\
   = 2\sqrt{17\mathcal{E}T\log(2|\F|T^3/\delta)}(\textup{poly}(\log T)+1),
\end{align}
where the first line uses the equivalence proved in Lemma \ref{lemma equivalence infinite action}; the second line is due to \eqref{eq: thm2 1}; the third line is due to  $\sum_{t=1}^T 1/\sqrt{t}\leq \sqrt{T}$; the fourth line is due to $\beta_T>\beta_t$ and $\beta_T>1$ when $T> \mathcal{E}$;  and the sixth line  is due to the condition \textup{\RNum{2}} in Assumption \ref{asm per context}.
By Azuma's inequality, with probability at least $1-\delta/2$, we can bound the regret by 
\begin{align}\label{eq: azuma infinite}
    \text{Regret}(T, \text{Algorithm \ref{alg: uccb ia}})\leq  \sum_{t=1}^T \E[f^*(x_t,\pi_{f^*}(x_t))-f^*(x_t,a_t)|H_{t-1}]+\sqrt{2T\log(2/\delta)}.
\end{align} Therefore, by a union bound and inequalities \eqref{eq: used in proof of thm infinite} \eqref{eq: azuma infinite}, with probability at least $1-\delta$, the regret of Algorithm $\ref{alg: uccb}$  after $T$ rounds is upper bounded by
\begin{align*}
   \textup{Regret}(T, \text{Algorithm \ref{alg: uccb ia}})\leq  2\sqrt{17\mathcal{E}T\log(2|\F|T^3/\delta)}(\textup{poly}(\log T)+1)+\sqrt{2T\log(2/\delta)}.
\end{align*}

Combine the case $T\leq \mathcal{E}$ and $T >\mathcal{E}$ we finish the proof.

\hfill$\square$

\paragraph{}

\begin{lemma}[\bf{counterfactual confidence bound}]\label{lemma uniform time infinite}Consider a non-randomized contextual bandit algorithm that  selects  $\pi_t$ based on $H_{t-1}$  and chooses the action $a_t=\pi_t(x_t)$ at all rounds $t$. Then $\forall \delta\in (0,1)$, with probability at least $1-{\delta}/2$, 
we have
\begin{align*}
  \big|\E_x[\widehat{f}_{t}(x,\pi(x)]-\E_x[f^*(x,\pi(x))]\big|\leq \E[1\land\beta_t{V_x(\pi(x)||\{\pi_i(x)\}_{i=1}^{t-1}) }] +\mathcal{E}\beta_t/t.
\end{align*}
uniformly over all $\pi\in\Pi$ and all $t\geq 2$.
\end{lemma}

\paragraph{Proof of Lemma \ref{lemma uniform time infinite}.}
For a fixed $f$, we denote $Y_{f,i}=(f(x_i,a_i)-r_i(x_i,a_i))^2-(f^*(x_i,a_i)-r_i(x_i,a_i))^2$, $i=1,2,\dots$. From Lemma \ref{lemma uniform convergence time},  $\forall \delta\in (0,1)$ , with probability at least $1-{\delta}/2$, 
we have
\begin{align}\label{eq: uniform convergence time regularization}
\sum_{i=1}^{t-1}\E_{x_i,a_i}\left[(f_{t}(x_i,a_i)-f^*(x_i,a_i))^2|H_{i-1}\right]\le  68\log({2|\F|t^3}/{\delta})+\sum_{i=1}^{t-1}Y_{f_t,i},
\end{align}
uniformly over all $t\geq 2$ and all fixed sequence $f_2, f_3, \dots f_{t+1}, \dots\in\F$. 

Use the fact that $\pi_i$ is completely determined by $H_{i-1}$ and independent with $x_i$, we obtain:
\begin{align}\label{eq: before expectation infinite action}
\sum_{i=1}^{t-1}\E_{x}\left[(\widehat{f}_{t}(x, \pi_i(x))-f^*(x,\pi_i(x)))^2\right]\le  68\log({2|\F|t^3}/{\delta})+\sum_{i=1}^{t-1}Y_{f_t,i}= 4\mathcal{E}\beta_t^2/t+\sum_{i=1}^{t-1}Y_{f_t,i},
\end{align}uniformly over all $t\geq 2$.

From the definition of counterfactual action divergence, we know
$\forall x\in\X$,
\begin{align*}
   {|{f}_{t}(x,\pi(x))-f^*(x,\pi(x))|}\leq {\sqrt{V_x(\pi(x)||\{\pi_i(x)\}_{i=1}^{t-1}) }}\sqrt{\sum_{i=1}^{t-1}({f}_{t}(x, \pi_i(x))-f^*(x,\pi_i(x)))^2}.
\end{align*}
Applying the AM-GM inequality to the above inequality, we obtain
\begin{align*}
    {|{f}_{t}(x,\pi(x))-f^*(x,\pi(x))|}\leq {\beta_t{V_x(\pi(x)||\{\pi_i(x)\}_{i=1}^{t-1}) }}+\frac{1}{4\beta_t}{\sum_{i=1}^{t-1}({f}_{t}(x, \pi_i(x))-f^*(x,\pi_i(x)))^2}.
\end{align*}
Since $\F$ is bounded by $[0,1]$, we further obtain
\begin{align}\label{eq: before expectation land}
     {|{f}_{t}(x,\pi(x))-f^*(x,\pi(x))|}\leq \max\left\{{\beta_t{V_x(\pi(x)||\{\pi_i(x)\}_{i=1}^{t-1}) }}+\frac{1}{4\beta_t}{\sum_{i=1}^{t-1}({f}_{t}(x, \pi_i(x))-f^*(x,\pi_i(x)))^2},1\right\}\nonumber\\
     \leq \max\left\{1, \beta_t{V_x(\pi(x)||\{\pi_i(x)\}_{i=1}^{t-1}) }\right\} +\frac{1}{4\beta_t}{\sum_{i=1}^{t-1}({f}_{t}(x, \pi_i(x))-f^*(x,\pi_i(x)))^2}\nonumber\\
     =1\land\beta_t{V_x(\pi(x)||\{\pi_i(x)\}_{i=1}^{t-1}) } +\frac{1}{4\beta_t}{\sum_{i=1}^{t-1}(\widehat{f}_{t}(x, \pi_i(x))-f^*(x,\pi_i(x)))^2}
\end{align}

By taking expectation on both side of \eqref{eq: before expectation land} and using \eqref{eq: before expectation infinite action}, we obtain that with probability at least $1-\delta/2$,
\begin{align*}
  \big|\E_x[{f}_{t}(x,\pi(x))-f^*(x,\pi(x))]\big|\leq \E_x|{f}_{t}(x,\pi(x))-f^*(x,\pi(x))|\\
  \leq \E_x[1\land\beta_t{V_x(\pi(x)||\{\pi_i(x)\}_{i=1}^{t-1}) }] +\frac{1}{4\beta_t}\E_x[{\sum_{i=1}^{t-1}({f}_{t}(x, \pi_i(x))-f^*(x,\pi_i(x)))^2}]\\
  =
  \E_x[1\land\beta_t{V_x(\pi(x)||\{\pi_i(x)\}_{i=1}^{t-1}) }] +\frac{1}{4\beta_t}{\sum_{i=1}^{t-1}\E[({f}_{t}(x, \pi_i(x))-f^*(x,\pi_i(x)))^2}]
  \\\leq  \E_x[1\land\beta_t{V_x(\pi(x)||\{\pi_i(x)\}_{i=1}^{t-1}) }] +\mathcal{E}\beta_t/t+\frac{1}{4\beta_t}\sum_{i=1}^{t-1}Y_{f_t,i}.
\end{align*}
uniformly over all $\pi\in\Pi$, all $t\geq 2$ and all fixed sequence $f_2,f_3,\dots,\in\F$. Here the first inequality is due to the triangle inequality; the second inequality is due to \eqref{eq: before expectation land}; the last inequality is due to \eqref{eq: before expectation infinite action}.

By taking $f_t=\widehat{f}_t$ the the least square solution that minimizes $\sum_{i=1}^{t-1} (f(x_i,a_i)-r_i(x_i,a_i))^2$, we have $Y_{f_t,i}\leq 0$ and finish the proof.

\hfill$\square$

\paragraph{}

\begin{lemma}\label{lemma equivalence infinite action}
 Consider an algorithm that choose policy  $\pi_t$ by
\begin{align*}
    \pi_t(x)=\left\{
\begin{aligned}
& A_{x,t} & \textup{ if }t\leq d_x, \\
  &\argmax_{\A(x)} \left\{ \widehat{f}_t(x,\cdot)+\beta_tV_{x_t}(\cdot||\{\pi({x})\}_{j=1}^{t-1})\right\} & \textup{ if }t>d_x.
\end{aligned}
\right.
\end{align*}
($A_{x,t}$ is determined the initialization oracle and the input $\A(x)$; the ``argmax" problem when $t>d_x$ is computed via the action maximization oracle.)
Then this algorithm produces the same actions as those produced by Algorithm \ref{alg: uccb ia}.
\end{lemma}
\paragraph{Proof of Lemma \ref{lemma equivalence infinite action}.}
The proof to this lemma is straightforward as the stated policy optimization problem is decomposable across contexts.

\hfill$\square$

\section{Proofs for the ``optimistic subroutine" in Section \ref{sec randomized}}

In this subsection we prove Proposition \ref{prop optimization}. Our proof is motivated by Agarwal et. al. [\citealp{agarwal2014taming}, Lemma 6, Lemma 7].

We call $q$ an improper distribution if:
1) $\int_{\A}q(a)\textup{d}a$ is within $(0,1]$ but not necessarily equal to one;  and 2) $\E_{{a}\sim q}[{a}{a}^\top]$ is invertible. We define the improper expectation $\E_{a\sim q} [W(a)]$ for any random variable $W:\A\rightarrow \R$ by the integral
$\int_{a\in\A} W(a) q(a)\textup{d}a$. 

We aim to minimize the potential function
\begin{align}
     \Phi(q):=-2\log(\det(\E_{{a}\sim q}[{b_a}{b_a}^\top]))+\E_{a\sim q}[2d+(\widehat{h}(\widehat{a})-\widehat{h}(a))/\beta],
\end{align}
 where $b_a$ is the coefficient vector of $a$ when the basis is the barycentric spanner $\{A_i\}_{i=1}^d$.  We prove that after each iteration, either Algorithm \ref{alg: coordinate}  outputs a desired distribution that satisfies both \eqref{eq: optimization generalized falcon 1} and \eqref{eq: optimization generalized falcon 2}, or 
\begin{align}\label{eq: support lemma opt}
    \Phi(q_t)\leq \Phi(q_{t-1})- \frac{1}{4}.
\end{align}
Since $\Phi$ function is bounded the algorithm must halt within finite iterations. \eqref{eq: support lemma opt} is a consequence of the following two lemmas:

\begin{lemma}\label{lemma non increasing}
When $\E_{a\sim q(a)}[2d+(\widehat{h}(\widehat{a})-\widehat{h}(a))/\beta]\geq 2d$, the objective $\Phi(q)$ will not increase if we multiply $q$ by $\frac{2d}{2d+\E_{a\sim q}[2d+(\widehat{h}(\widehat{a})-\widehat{h}(a))/\beta]}$. That is, after step \eqref{eq: backtraking} in Algorithm \ref{alg: coordinate}, we always have
\begin{align*}
    \Phi(q_{t-\frac{1}{2}})\leq \Phi(q_{t-1}).
\end{align*}
\end{lemma}

\begin{lemma}\label{lemma coordinate descent}
If Algorithm \ref{alg: coordinate} does not halt at round $t$, then after the coordinate descent step \eqref{eq: coordinate descent step} in Algorithm \ref{alg: coordinate}, we always have
\begin{align*}
    \Phi(q_t)\leq \Phi(q_{t-\frac{1}{2}})-\frac{1}{4}.
\end{align*}
\end{lemma}

Now we present the proof of Proposition \ref{prop optimization}, as well as  proofs of Lemma \ref{lemma non increasing} and Lemma \ref{lemma coordinate descent}. 

\paragraph{}\ 

\noindent\textbf{Proof of Proposition \ref{prop optimization}. }
From Lemma \ref{lemma non increasing} and Lemma \ref{lemma coordinate descent} we know that if the algorithm does not halt at round $t$, then $\Phi(q_t)\leq \Phi(q_{t-\frac{1}{2}})-\frac{1}{4}$. Assume Algorithm \ref{alg: coordinate} does not halt after $t$ rounds. Then we have
\begin{align}
   2d\log d+ 2d+\frac{1}{\beta}\geq 2d\log d+\E_{a\sim q_0}[2d+(\widehat{h}(\widehat{a})-\widehat{h}(a))/\beta]=\Phi(q_0)\nonumber\\\geq \Phi(q_{t})+\frac{t}{4}= -2\log(\det(\E_{a\sim q_t}[b_ab_a^\top]))+\E_{a\sim q_t}[2d+(\widehat{h}(\widehat{a})-\widehat{h}(a))/\beta]+\frac{t}{4}\nonumber\\
    \geq -2\log(\det(\E_{a\sim q_t}[b_ab_a^\top]))+\frac{t}{4}. \label{eq: bound potential randomized}
\end{align}
where the first inequality is due to  $\int_{A}q_0(a)\textup{d}a=1$; the first equation is due to $-2\log(\det(\E_{a\sim q_0}[b_ab_a^\top]))=-2\log(\det(\frac{1}{d}I))=2d\log d$; and the last inequality is due to $\E_{a\sim q_t}[2d+(\widehat{h}(\widehat{a})-\widehat{h}(a))/\beta]\geq0$.

Since the initialization actions consist of a barycentric spanner of $\A$, all coordinates of $b_a$ is within $[-1,1]$, $\forall a\in\A$. Clearly $\|b_a\|\leq \sqrt{d}$ for all $a\in\A$. We know that $q_t$ is a improper distribution with at most $d+t$ non-zero supports, so we assume $q_t=\sum_{i=1}^{d+t}q_t(A_i)\mathds{1}_{A_i}$.
\begin{align*}
\det(\E_{a \sim q_t}[b_ab_a^\top])=\det(\sum_{i=1}^{d+t}q_t(A_i)b_{A_i}b_{A_i}^\top)\\\leq (\frac{\text{tr}(\sum_{i=1}^{d+t}q_t(A_i)b_{A_i}b_{A_i}^\top)}{d})^d=(\frac{\sum_{i=1}^{d+t}q_t(A_i)\text{tr}(b_{A_i}b_{A_i}^\top)}{d})^d\\
\leq (\frac{\sum_{i=1}^{d+t}q_t(A_i)\|b_{A_i}\|^2}{d})^d\leq 1,
\end{align*}
where the first inequality is due to the AM-GM inequality; the last inequality is due to $\|b_a\|\leq \sqrt{d}$ for all $a\in\A$  and $\sum_{i=1}^{d+t}q_t(A_i)=\int_{\A}q_t(a)da\leq 1$. As a result, we obtain $\log(\det(\E_{a\sim q_t}[b_ab_a^\top]))\leq 0$.
Combine this result with \eqref{eq: bound potential randomized}, we obtain
\begin{align*}
   t\leq  8d(\log d+1)+\frac{4}{\beta}.
\end{align*}
So  Algorithm \ref{alg: coordinate} must halt within at most $\lceil \frac{4}{\beta_m}+8d(\log d+1) \rceil$ iterations. When it halts, it is straightforward to verify that the output distribution is proper and satisfies both \eqref{eq: optimization generalized falcon 1} and \eqref{eq: optimization generalized falcon 2}.

\hfill$\square$

\paragraph{}\ 

\noindent\textbf{Proof of Lemma \ref{lemma non increasing}. } Denote $w(a)=(\widehat{h}(\widehat{a})-\widehat{h}(a))/\beta$. Given an arbitrary improper distribution $q$, we view $\Phi(c\cdot q)$ as a function on the scaling factor $c$. By the chain rule, we can compute the derivative of this function with respect to $c$,
\begin{align*}
    \partial_c \Phi(c\cdot q)=\int_{\A} \Big[\partial_{c q(a)} \Phi (c\cdot q)\Big]\Big(\partial_{c} cq(a)\Big) \textup{d}a\\
    = \int_{\A} \Big[-2a^\top(\E_{\widetilde{a}\sim  cq} [\widetilde{a}\widetilde{a}^\top])^{-1}a +2d+w(a)\Big] q(a)\textup{d}a\\
    =\frac{2}{c}\E_{{a}\sim q}[a^\top(\E_{{a}\sim q} [{a}{a}^\top])^{-1}a ]+2d+\E_{{a}\sim q}[w({a})],
\end{align*}
where the second equation use the fact that the partial gradient of $\log(\det(\E_{a\sim c q}[b_ab_a^\top] ))$ with respect to the coordinate $cq(a)$ is $b_a^\top (\E_{\widetilde{a}\sim  cq} [b_{\widetilde{a}} b_{\widetilde{a}}^\top])^{-1}b_a=a^\top (\E_{\widetilde{a}\sim  cq} [\widetilde{a}\widetilde{a}^\top])^{-1}a$.

By the ``trace trick", we have
\begin{align*}
    \E_{a\sim q}[a^T(\E_{a\sim q} [aa^T])^{-1}a ]=\text{tr}(\E_{a\sim q}[a^T(\E_{a\sim q} [aa^T])^{-1}a ])\\
    =\E_{a\sim q}[\text{tr}(a^T(\E_{a\sim q} [aa^T])^{-1}a )]\\
    =\E_{a\sim q}[\text{tr}(aa^T(\E_{a\sim q} [aa^T])^{-1} )]\\
    =\text{tr}(\E_{a\sim q}aa^T(\E_{a\sim q} [aa^T])^{-1} ))=d.
\end{align*}
So we have 
\begin{align*}
    \partial_c \Phi(c\cdot q)=-\frac{2d}{c}+2d+\E_{a\sim q}[w(a)].
\end{align*}
This means that for all $c\in [\frac{2d}{2d+\E_{a\sim q}[w(a)]},1]$, $\partial_c \Phi(c\cdot q)\geq 0$.

\paragraph{}\ 

\noindent\textbf{Proof of Lemma \ref{lemma coordinate descent}. }
We define $$\Delta_t=\frac{-2a_t^\top(\E_{{a}\sim q}[{a}{a}^\top])^{-1}a_t+2d+(\widehat{h}(\widehat{a})-\widehat{h}(a_t))/\beta}{ (a_t^\top(\E_{{a}\sim q}[{a}{a}^\top])^{-1}a_t)^2},$$ then the coordinate descent step \eqref{eq: coordinate descent step} is $q_t=q_{t-\frac{1}{2}}+\Delta_t\mathds{1}_{a_t}$.
\begin{align}
    \Phi(q_{t-\frac{1}{2}})-\Phi(q_t)=2\log(\frac{\det(\E_{{a}\sim q^+}[{b_a}{b_a}^\top])}{\det(\E_{{a}\sim q}[{b_a}{b_a}^\top])}-\Delta_t\big(2d+(\widehat{h}(\widehat{a})-\widehat{h}(a_t))/\beta\big)\nonumber\\
    = 2\log(1+{\Delta_t a_t^\top(\E_{{a}\sim q}[{a}{a}^\top])^{-1}a_t})-\Delta_t\big(2d+(\widehat{h}(\widehat{a})-\widehat{h}(a_t))/\beta\big)\label{eq: use matrix determinant lemma}\\
    \geq 2{a_t^\top(\E_{{a}\sim q}[{a}{a}^\top])^{-1}a_t}\Delta_t-({a_t^\top(\E_{{a}\sim q}[{a}{a}^\top])^{-1}a_t})^2\Delta_t^2-\Delta_t\big(2d+(\widehat{h}(\widehat{a})-\widehat{h}(a_t))/\beta\big)\label{eq: use log}\\
    =\frac{\Big(-2a_t^\top(\E_{{a}\sim q}[{a}{a}^\top])^{-1}a_t+2d+(\widehat{h}(\widehat{a})-\widehat{h}(a_t))/\beta\Big)^2}{4 (a_t^\top(\E_{{a}\sim q}[{a}{a}^\top])^{-1}a_t)^2},\label{eq: use Delta}
\end{align}
where \eqref{eq: use matrix determinant lemma} is due to the matrix determinant lemma as well as $a_t^\top(\E_{{a}\sim q}[{a}{a}^\top])^{-1}a_t=b_{a_t}^\top(\E_{{a}\sim q}[{b_a}{b_a}^\top])^{-1}b_{a_t}$; \eqref{eq: use log} is due to the inequality $\log(1+w)\geq w-\frac{w^2}{2}$ for all $w\geq 0$; \eqref{eq: use Delta} is due to the definition of $\Delta_t$ which maximize the quadratic function in \eqref{eq: use log}.

As the algorithm does not halt, we have $a_t^\top(\E_{{a}\sim q}[{a}{a}^\top])^{-1}a_t\geq 2d+(\widehat{h}(\widehat{a})-\widehat{h}(a_t))/\beta$, so 
\begin{align*}
    |-2a_t^\top(\E_{{a}\sim q}[{a}{a}^\top])^{-1}a_t+2d+(\widehat{h}(\widehat{a})-\widehat{h}(a_t))/\beta|\\
    =2a_t^\top(\E_{{a}\sim q}[{a}{a}^\top])^{-1}a_t-2d-(\widehat{h}(\widehat{a})-\widehat{h}(a_t))/\beta\\\geq a_t^\top(\E_{{a}\sim q}[{a}{a}^\top])^{-1}a_t.
\end{align*}
Combine this inequality with \eqref{eq: use Delta} we obtain
\begin{align*}
    \Phi(q_{t-\frac{1}{2}})-\Phi(q_t)\geq \frac{1}{4}.
\end{align*}

\hfill$\square$

\end{document}